\pdfoutput=1
\documentclass{siamltex}

\usepackage{amsmath}
\usepackage{amssymb}
\usepackage{bm}
\usepackage{algorithm}
\usepackage{algpseudocode}

\usepackage{graphicx}
\usepackage{caption}
\usepackage{subcaption}

\usepackage{natbib}

\DeclareMathOperator*{\argmin}{arg\,min}

\DeclareMathOperator*{\mari}{MARI}

\title{On latent position inference from doubly stochastic messaging activities}

\author{
Nam H.\ Lee\footnotemark[2]\ \footnotemark[3] \and 
Jordan Yoder\footnotemark[2] \footnotemark[6] \and
Minh Tang\footnotemark[2]\ \footnotemark[4] \and 
Carey E. Priebe\footnotemark[2]\ \footnotemark[5]
}

\begin{document}

\maketitle
\renewcommand{\thefootnote}{\fnsymbol{footnote}}
\footnotetext[2]{Department of Applied Mathematics and Statistics,  
Johns Hopkins University, Baltimore, MD 21218, USA.}

\footnotetext[3]{nhlee@jhu.edu}
\footnotetext[6]{jyoder6@jhu.edu}
\footnotetext[4]{mtang10@jhu.edu}
\footnotetext[5]{cep@jhu.edu}

\renewcommand{\thefootnote}{\arabic{footnote}}

\begin{abstract}
We model messaging activities as a hierarchical doubly stochastic point
process with three main levels, and develop an iterative algorithm for
inferring actors' relative latent positions from a stream of messaging
activity data.  Each of the message-exchanging actors is modeled as a process in a
latent space.  The actors' latent positions are assumed to be influenced by
the distribution of a much larger population over the latent space.  Each
actor's movement in the latent space is modeled as being governed by two
parameters that we call confidence and visibility, in addition to dependence on
the population distribution.  The messaging frequency between a pair of actors
is assumed to be inversely proportional to the distance between their latent
positions.  Our inference algorithm is based on a projection approach to an
online filtering problem.  The algorithm associates each actor with a
probability density-valued process, and each probability density is assumed
to be a mixture of basis functions.  For efficient numerical experiments, we
further develop our algorithm for the case where the basis functions are
obtained by translating and scaling a standard Gaussian density.
\end{abstract}



\pagestyle{myheadings}
\thispagestyle{plain}
\markboth{N.\ H.\ Lee, J.\ Yoder, M.\ Tang and C.\ E.\ Priebe}{On latent
position inference from doubly stochastic messaging activities}

%\begin{center}
%{\footnotesize Source Filename: \texttt{\currfilename}}\\
%{\footnotesize Last Compiled on \today}
%\end{center}

\section{Introduction}
Communication networks are presenting ever-increasing challenges in a wide
range of applications, and there is great interest in inferential methods for
exploiting the information they contain.  A common source of such data is a
corpus of time-stamped messages such as e-mails or SMS (short message service).
Such messaging data is often useful for inferring a social structure of the
community that generates the data.  In particular, messaging data is an asset
to anyone who would like to cluster actors according to their
\emph{similarity}.  A practitioner is often privy to messaging data in a
\emph{streaming} fashion, where the word \emph{streaming} describes a
practical limitation, as the practitioner might be privy only to the incoming
data in a fixed summarized form without any possibility to retrieve past
information.  It is in the practitioner's interest to transform the
summarized data so that the transformed data is appropriate for detecting
\emph{emerging} social trends in the source community.

We mathematically model such streaming data as a collection of tuples of the
form $\mathcal{D} = \{(t_\ell, {i_\ell,j_\ell})\}$ of time and actors, where
$i_\ell$ and $j_\ell \in \{1,\ldots, n\}$ represent actors exchanging the
$\ell$-th message and $t_\ell \in \mathbb{R}_{+}$ represents the occurrence
time of the $\ell$-th message.  There are many models suitable for dealing
with such data.  The most notable are the Cox hazard model, the doubly
stochastic process (also known as the Cox process), and the self-exciting
process (although self-exciting processes are sometimes considered as special
cases of the Cox hazard model).  For references on these topics, see
\cite{Andersen1995}, \cite{DonaldSnyder75} and \cite{Bremaud1981}.  
All three models are related to each other; however, the distinctions are
crucial to statistical inference as they stem from different assumptions on
information available for (online) inference.  To transform $\mathcal D$ data
to a data representation more suitable for clustering actors, we model
$\mathcal D$ as a (multivariate) doubly stochastic process, and develop a
method for embedding $\mathcal D$ as a stochastic process taking values in
$\mathbb R^d$ for some suitably chosen $d \in \mathbb N$. 

\section{Related works}
For statistical inference when there is information available beyond
$\mathcal D$, the Cox-proportional hazard model is a natural choice.  In
\cite{Heard2010} and \cite{PerryWolfe2010}, for instance, instantaneous
intensity of messaging activities between each pair of actors is assumed to be
a function of, in the language of generalized linear model theory, known
covariates with unknown regression parameters.  More specifically, in
\cite{Heard2010}, the authors consider a model where $\lambda_{ij}(t) =
A_{ij}(t) (B_{ij}(t) + 1)$ with $A_{ij}(t)$ and $B_{ij}(t)$ representing
independent counting processes, e.g., $A_{ij}(t)$ are Bernoulli random
variables and $B_{ij}(t)$ are random variables from the exponential family. On
the other hand, in \cite{PerryWolfe2010}, a Cox multiplicative model was
considered where $\lambda_{ij}(t) = \xi_{i}(t) \exp\{\beta_{0}^{T}
x_{ij}(t)\}$. The model in \cite{PerryWolfe2010} posits that actor $i$
interacts with actor $j$ at a baseline rate $\xi_{i}$ modulated by the pair's
covariate $x_{ij}$ whose value at time $t$ is known and $\beta_0$ is a common
parameter for all pairs.  In \cite{PerryWolfe2010}, it is shown under some
mild conditions that one can estimate the global parameter $\beta_0$
consistently.  In \cite{StomakhinShortBertozzi}, the intensity is modeled for
\emph{adversarial} interaction between \emph{macro} level groups, and a
problem of nominating unknown participants in an event as a missing data
problem is entertained using a self-exciting point process model.  In
particular, while no explicit intensity between a pair of actors (gang
members) is modeled, the event intensity between a pair of groups (gangs) is
modeled, and the spatio-temporal model's chosen intensity process is
self-exciting in the sense that each event can affect the intensity
process.

When data $\mathcal D$ is the only information at hand, a common approach is
to construct a time series of (multi-)graphs to model association among
actors.  For such an approach, a simple method to obtain a time series of
graphs from $\mathcal{D}$ is to ``pairwise threshold'' along a sequence of
non-overlapping time intervals.  That is, given an interval, for each pair of
actors $i$ and $j$, an edge between vertex $i$ and vertex $j$ is formed if the
number of messaging events between them during the interval exceeds a certain
threshold.  This is the approach taken in
\cite{cortes03:_comput,eckmann04:_entrop,adamic05:_how}, \cite{JDLeeMaggioni}
and \cite{RanolaLange2010}, to mention just  a few examples.  The resulting
graph representation is often thought to capture the structure of  some
underlying social dynamics.  However, recent empirical research, e.g.,
\cite{DeChoudhury2010}, has begun to challenge this approach by noting that
changing the thresholding parameter can produce dramatically different graphs.  

Another useful approach when $\mathcal D$ is the only information available is
to use a doubly stochastic process model in which count processes are driven
by latent factor processes.  This is the approach taken explicitly in
\citet{LP-SISP2011} and \citet{TPLP2011}, and this is also done implicitly in
\cite{ChiKolda2012}.  In \citet{LP-SISP2011} and \citet{TPLP2011} interactions
between actors are specified by proximity in their latent positions;
the closer two actors are to each other in their latent configuration, the
more likely they exchange messages.  Using our model, we consider a problem of
clustering actors ``online'' by studying their messaging activities.  This
allows us a more geometric approach afforded by embedding $O(n^2)$ data to an
$O(n\times d)$ representation for some fixed dimension $d$.

In this paper, we propose a useful mathematical formulation of the problem as a
filtering problem based on both a multivariate point process observation
and a population latent position distribution.

\section{Notation}
As a convention, we assume that a vector is a column vector if its dimension
needs to be disambiguated. We denote by $\mathcal F_t$ the filtration up to
time $t$ that models the information generated by undirected
communication activities between actors in the community, where ``undirected''
here means we do not know which actor is the sender and which is the receiver.
We denote by $\mathcal M_1(\mathbb R^d)$ the space of probability measures on
$\mathbb R^d$.  For a probability density function defined on $\mathbb R^d$, $\phi(x;c, s)$
denotes the probability density function that is proportional to
$\phi(s^{-1}(x - c))$ where the normalizing constant does not depend on $x$.   
The set of all $r\times c$ matrices over the
reals is denoted by $\mathbb M_{r,c}$.  For each $k_1\times k_2$ matrix $M \in
\mathbb M_{k_1,k_2}$, we write $\|M\|_F = (\sum_{r=1}^{k_1} \sum_{c=1}^{k_2}
M_{rc}^2)^{1/2}$. Given a vector $v \in \mathbb R^d$, we write $\|v\|$ for 
its Euclidean norm.  Let $\mathbb R_+= (0,\infty)$ and $\mathbb M_{k_1,k_2}^+ :=
\{ M \in \mathbb M_{k_1,k_2}: M_{r,c} \ge 0\}$.  For each $M_1$ and $M_2 \in
\mathbb M_{k_1,k_2}$, we write $M_1 * M_2$ for  the Hadamard product of $M_1$
and $M_2$, i.e., $*$ denotes component-wise multiplication.  Given vectors
$v_1,\ldots,v_n$ in $\mathbb R^d$, the Gram matrix of the ordered collection
$v=(v_1,\ldots,v_n)$ is the $d\times d$ matrix $G$ such that its $(r,c)$-
entry $G_{r,c}$ is the inner product $\langle v_r, v_c \rangle = v_r^\top v_c$
of $v_r$ and $v_c$.  Given a matrix $M \in \mathbb M_{d,d}$, $\diag(M)$ is the
column vector whose $k$-th entry is the $k$-th diagonal element of $M$.  With
a slight abuse of notation, given a vector $v \in \mathbb R^d$, we will also
denote by $\diag(v)$ the $d\times d$ diagonal matrix such that its $k$-th
diagonal entry is $v_k$.  We always use $n$ for the number of actors under
observation and $d$ for the dimension of the latent space.  We denote by
$\mathbb X$ the $n$-fold product $\mathbb R^d \times \cdots \times \mathbb
R^d$ of $\mathbb R^d$. An element of $\mathbb X$ will be written in bold face
letters, e.g.~$\bm x \in \mathbb X$.  Similarly, bold faced letters will
typically be used to denote objects associated with the $n$ actors
collectively.  An exception to this convention is the identity matrix which is
denoted by $\bm I_d=\bm I$, where the dimension is specified only if needed
for clarification. Also, we write $\bm 1$ as the column matrix of ones.  With
a bit of abuse of notation, we also write $\bm 1$ for an indicator function,
and when confusion is possible, we will make our meaning clear. 

\section{Hierarchical Modeling}
Our actors under observation are assumed to be a subpopulation of a bigger
population. That is, we observe  actors $\{1,\ldots,n\}$ that are sampled from
a population for a longitudinal study.  We are not privy to the
actors' latent features that determine the frequency of pairwise messaging
activities, but we do observe messaging activities $\mathcal{D}_t =\{
(t_\ell,i_\ell,j_\ell): t_\ell \le t\}$.  A notional illustration of our
approach thus far is summarized in Figure \ref{fig:hierarchalmodeldiagram},
Figure \ref{fig:subpop_historgrams}, and Figure \ref{fig:Kullback-Leibler
divergence - Simulation}.  In both Figure \ref{fig:subpop_historgrams} and
Figure \ref{fig:Kullback-Leibler divergence - Simulation}, $\tau_1$ represents
the (same) initial time when there was no cluster structure, and $\tau_2$ and
$\tau_3$ represent the emerging and fully developed latent position clusters
which represent the object of our inference task.  

\begin{figure}[!ht]
        \centering
        \includegraphics[height=0.4\textheight]{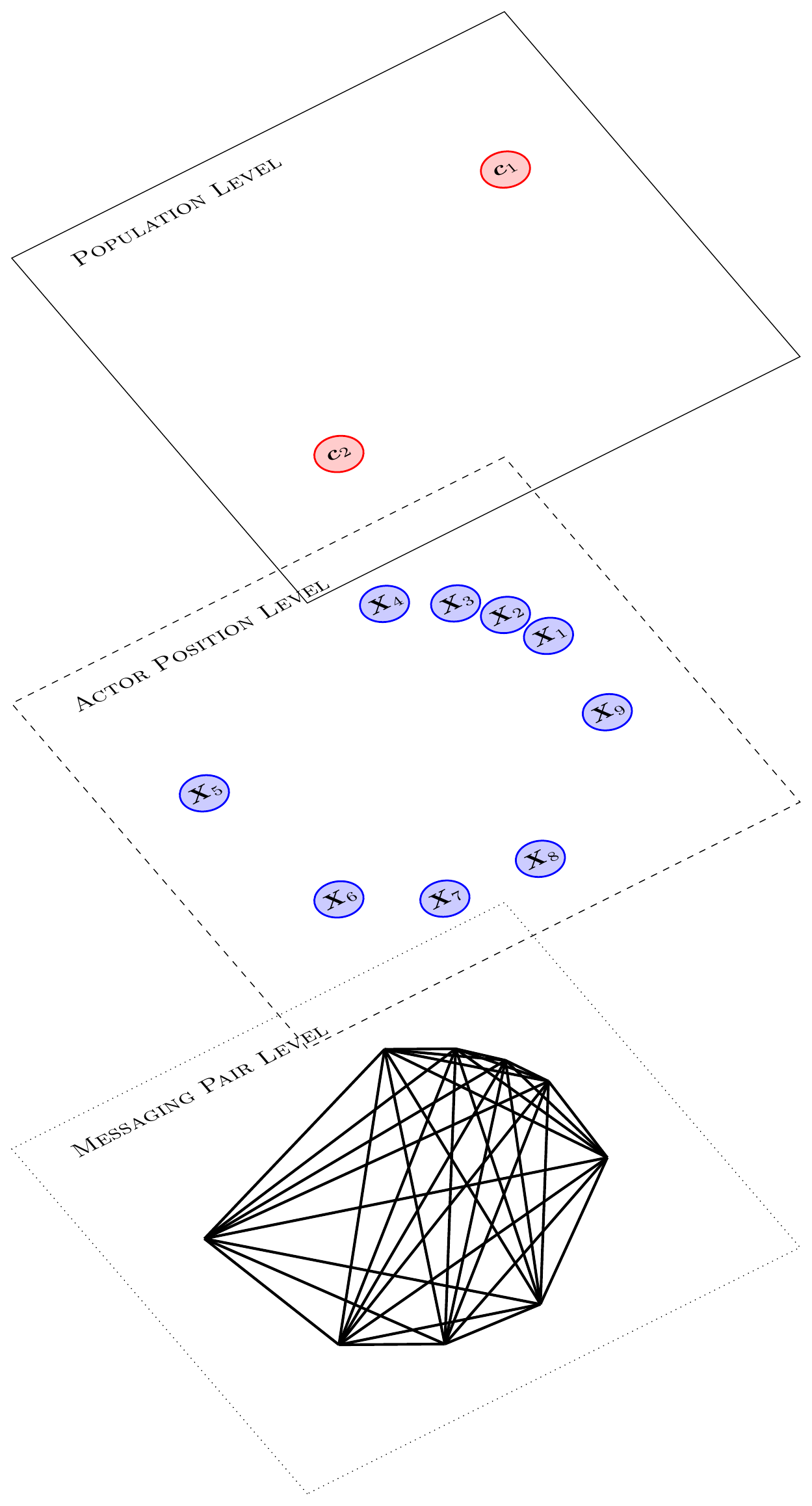}
    \caption{Hierarchical structure of the model.  The top, middle and bottom
    layers correspond to the population level, the actor level
    and the messaging level, respectively.  The top two levels, i.e., the population and the
    actor levels, are hidden and the bottom level, i.e., the messaging level, is 
    observed. See also Figure \ref{fig:hierarchalmodeldiagramwithparam} for
a more detailed diagram. }
    \label{fig:hierarchalmodeldiagram}
\end{figure}

\paragraph{Population density process level}
The message-generating actors are assumed to be members of a community, which
we call the population.  The aspect of the population that we model in this
paper is its members' distribution over a latent space in which the proximity
between a pair of actors determines the likelihood of the pair exchanging
messages.  The population distribution is to be time-varying and a mixture of
component distributions.

The latent space is assumed to be $\mathbb R^d$ for some $d \in \mathbb N$,
and the population distribution at time $t$ is assumed to have 
a continuous density $\mu_t$.  To be more precise, 
we assume that the (sample) path $t\rightarrow\mu_t$ is such 
that for each $y \in \mathbb R^d$,
\begin{align*}
    \mu_t(y) = \sum_{\ell} q_{t,\ell} \phi\left(  y;  c_{t,\ell}, \alpha_{t,\ell}\right),
\end{align*}
where
\begin{itemize}
    \item $q_{t,\ell}$ is a smooth \emph{sample} path of a stationary
        (potentially degenerate) diffusion process taking values in $(0,1)$,
    \item $\phi$ is a probability density function on $\mathbb R^d$ with
        convex support with its mean vector being the zero vector and its
        covariance matrix being a positive definite (symmetric)
        matrix, 
    \item $c_{t,\ell}$ is a smooth \emph{sample} path of an
        $\mathbb R^d$-valued (potentially non-stationary or degenerate)
        diffusion process,
    \item $\alpha_{t,\ell}$ is a smooth \emph{sample} path of a stationary
        (potentially degenerate) diffusion process taking values in
        $(0,\infty)$.
\end{itemize}
Note that it is implicitly assumed that $\sum_\ell q_{t,\ell} = 1$, and
additionally, we also assume that for each $k =1,\ldots, d$ and $m \in
\mathbb N$, the $m$-th moment $\langle \chi_k^m, \mu_t\rangle$ of the $k$-th 
coordinate of $\mu_t$ is finite, i.e., 
$\langle \chi_k^m, \mu_t \rangle := \int x_k^m \mu_t(x) dx < \infty$.

In this paper, we take $q_{t,\ell}, c_{t,\ell}$ and $\alpha_{t,\ell}$ as
exogenous modeling elements. However, for an example of a model with yet
further hierarchical structure, one \emph{could} take a cue from a continuous
time version of the classic ``co-integration'' theory, e.g., see
\cite{Comte1999}.  The idea is that the location $c_\ell$ of the $\ell$-th
center is non-stationary, but the inter-point distance between a combination
of the centers is stationary.  More specifically, one \emph{could} further
assume that there exist $d \times (d-d_0)$ matrix $\alpha_{\bot}$, $d \times
d_0$ matrix $\alpha$ and $d_0 \times d$ matrix $\beta$ such that 
\begin{itemize}
    \item $(\alpha_{\bot})^\top \alpha$ is the $(d-d_0)\times d_0$ dimensional zero matrix, 
    \item $(\alpha_{\bot})^\top c_{t,m}$ is a $(d-d_0)$ dimensional Brownian motion,
    \item $\beta^{\top} c_{t,m}$ is a stationary diffusion process.
\end{itemize}
Thus, the position of centers are unpredictable, but the relative distance
between each pair of centers are as predictable as that of a stationary
process. 
\begin{algorithm}
    \caption{Simulating a sample path of a population density process}
    \label{algo:PopulationProcess}
\begin{algorithmic}[1]
    \Require $((\bm q(t), \bm c(t), \bm \alpha(t)): t \in [0,T])$ and $\Delta t$
    \vskip0.1in
    \Procedure {PopulationProcess}{}
    \State $t \leftarrow 0$
    \While{$t < T$}
    \State $\mu_t(x) \leftarrow \sum_m {q_{t,m}} \phi(x ; c_{t,m}, \alpha_{t,m})$
        \State $t \leftarrow t + \Delta t$
    \EndWhile
    \EndProcedure
\end{algorithmic}
\end{algorithm}

\begin{figure}[!ht]
    \centering
    \begin{subfigure}[b]{0.3\textwidth}
        \centering
        \includegraphics[width=\textwidth]{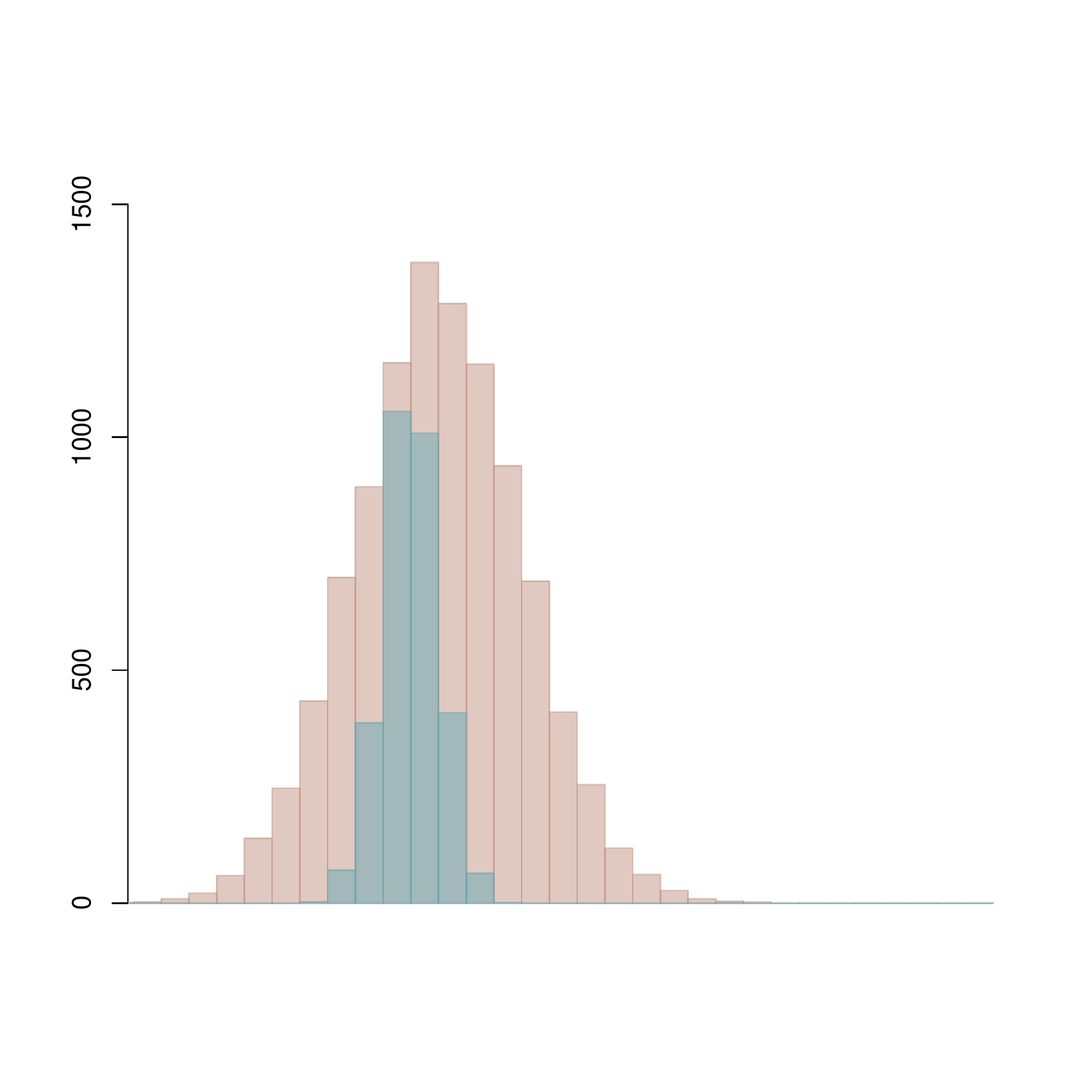}
        \caption{time $\tau_1$}
    \end{subfigure}
    \begin{subfigure}[b]{0.3\textwidth}
        \centering
        \includegraphics[width=\textwidth]{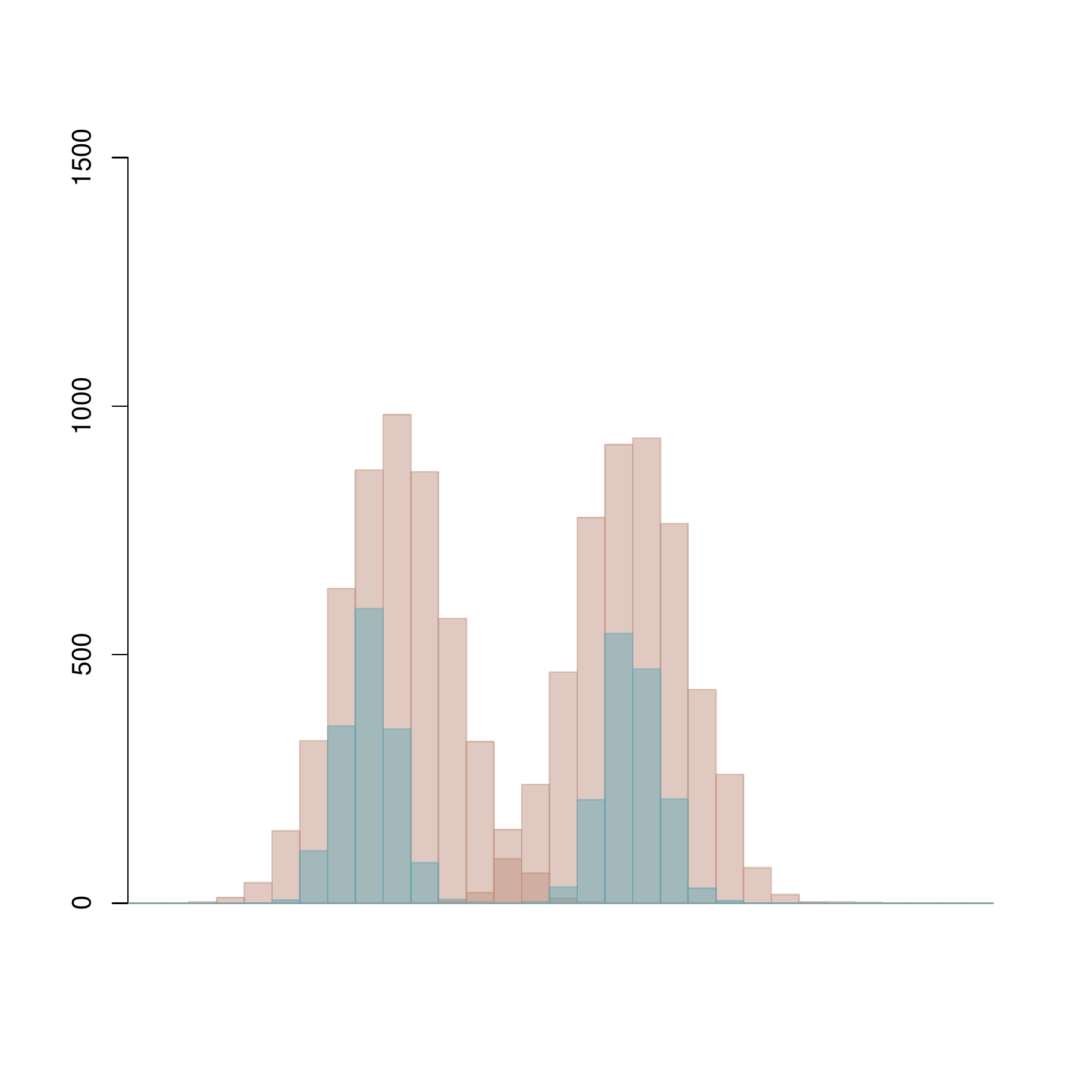}
        \caption{time $\tau_2$}
    \end{subfigure}
    \begin{subfigure}[b]{0.3\textwidth}
        \centering
        \includegraphics[width=\textwidth]{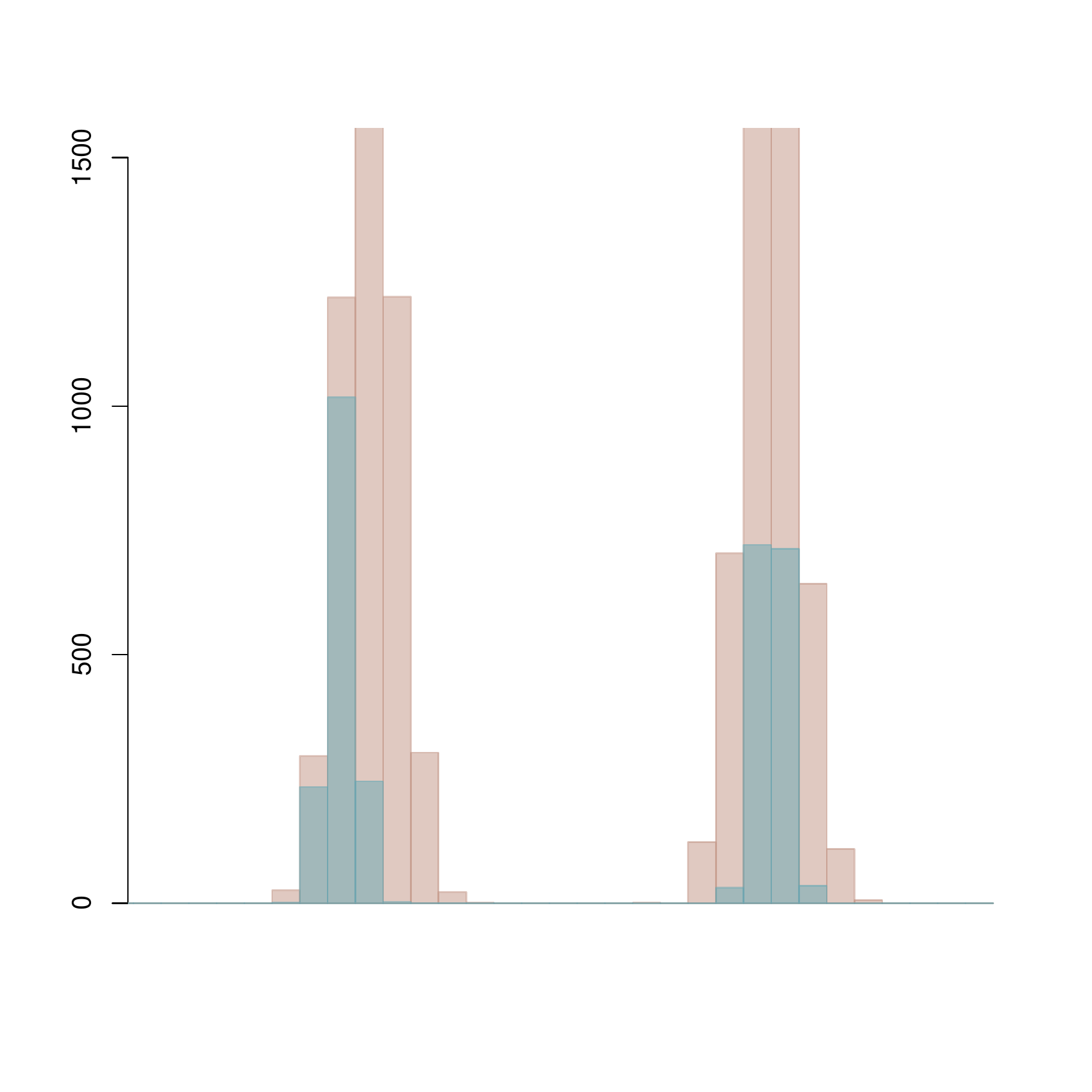}
        \caption{time $\tau_3$}
    \end{subfigure}
    \caption{A notional depiction of the evolution of the full population and
    subpopulation latent position distributions.  At each time $\tau_\ell$,
    the lightly-colored outer histogram represents the latent position
    distribution $\mu_t$ for the full population, and the darkly-colored inner
    histogram represents the distribution of the latent positions of actors
    under \emph{consideration}. The illustrated temporal order is $\tau_1 <
    \tau_2 < \tau_3$.
    }\label{fig:subpop_historgrams}
\end{figure}

\paragraph{Actor position process level}
Figure \ref{fig:subpop_historgrams} sketches the connection between actors and
populations.  We first define a process for a single actor.  To begin, for
each $t$, given $\mu_t$ and  $(\omega_t,\sigma_t) \in (0,1)\times (0,\infty)$,
we write
\begin{align*}
    \mathcal A_t f( x)
    = 
    \sum_{k=1}^d b_t^{k}( x) \frac{\partial }{\partial x_k} f( x) 
    +
    \sum_{k_1,k_2=1}^d a_t^{k_1,k_2}( x) \frac{\partial^2}{\partial x_{k_1}
    \partial x_{k_2}} f( x),
\end{align*} 
where
\begin{align*}
    \psi( z) &= \phi( z)/\phi( 0),\\
    b_t^{k}( x) &= 2(1-\omega_t)\int \psi\left(\dfrac{ y- x}{\sigma_t}\right) (y_k - x_k) \mu_t( y) d y,\\
    a_t^{k\ell}( x) &= (1-\omega_t)^2 \int \psi\left(\dfrac{ y-
    x}{\sigma_t}\right) (y_k- x_k)(y_\ell- x_\ell) \mu_t( y) d y.
\end{align*}
The formulation here for the $b_t^k$ and $a_t^{k\ell}$ is based on a quadratic
Taylor-series approximation of a so-called ``bounded confidence'' model
studied in \cite{GomezGrahamLeBoudec}.  Here, the value of $\omega_{t,i}$
represents the confidence level of actor $i$ on its current position and
$\sigma_{t,i}$ represents the visibility of other actors' position by actor
$i$.  Roughly speaking, an actor with a \emph{small} value of $\omega_{t,i}$ and
a \emph{large} value of $\sigma_{t,i}$ will be influenced \emph{greatly} by actors
that are positioned \emph{both near and far} in the latent space whereas an
actor with a \emph{large} value of $\omega_{t,i}$ and a \emph{small} value of
$\sigma_{t,i}$ will be influenced only a \emph{small} amount by actors that are
\emph{nearby} in the latent space. For further discussion on our motivation
for the form of $\mathcal A_t$, see Appendix \ref{appendix:motivationforA}.

For each actor $i$, the deterministic path $t \rightarrow
(\omega_{t,i},\sigma_{t,i})$ is assumed to be continuous, taking values in a
compact subset of $(0,1)\times (0,\infty)$.  It is assumed that given $t
\rightarrow \mu_t$, each actor's latent position process $X_i = (X_{t,i} :
t\in [0,T])$ is a diffusion process whose generator is $\mathcal A_t$, and
moreover we assume that $X_1,\ldots, X_n$ are mutually independent.  For each
$t$, let 
\begin{align*} 
    \bm X_t \equiv (X_{t,1},\ldots,X_{t,n})^\top,
\end{align*} 
where each $X_{t,i}$ is assumed to be a column vector, i.e., a
$d\times 1$ matrix. In other words, the $i$-th row of $\bm X_t$ is the
transpose of $X_{t,i}$.  

\begin{algorithm}[t!]
    \caption{Simulating a single actor's latent location process}
    \label{algo:LatentProcess}
\begin{algorithmic}[1]
    \Require $\Delta t$, $( (\omega_t,\sigma_t): t \in [0,T])$, and $(\mu_t: t \in [0,T])$
    \vskip0.1in
    \Procedure {LatentProcess}{}
    \State Compute $ b_t( x)$ and $ a_t( x)$
    \State Compute the non-negative definite symmetric square root $\sqrt{a_t( x)}$ 
    of $ a_t(x)$
    \State $t \leftarrow 0$
    \While{$t < T$}
        \State $\Delta W(t) \leftarrow$ {\sc StandardNormalVector}
        \State $X(t+\Delta t) = X(t) +  { b}_t(X(t)) \Delta t 
        + \sqrt{{ a}_t(X(t))} \Delta W(t) \sqrt{\Delta t}$
        \State $t \leftarrow t + \Delta t$
    \EndWhile
    \EndProcedure
\end{algorithmic}
\end{algorithm}

\paragraph{Messaging process level}
Denote by $N_{t,i \rightarrow j}$ the number of messages sent \emph{from}
actor $i$ \emph{to} actor $j$. Also, denote by $N_{t,ij}$ the number of
messages exchanged \emph{between} actor $i$ \emph{and} actor $j$.  
Note that $N_{t,ij} = N_{t,i\rightarrow j} + N_{j\rightarrow i}(t)$.  For each
actor $i$, we assume that the path $t \rightarrow \lambda_{t,i}$ is
deterministic, continuous and takes values in $(0,\infty)$.  For each $t$, we
assume that 
\begin{align*}
    \mathbb P[N_{t+dt,i\rightarrow j} = N_{t,i\rightarrow j} + 1 | \mathcal
    F_t, \bm X_s, s\le t]
    &= 
    (\lambda_{t,i} \lambda_{t,j}/2) p_{t,i\rightarrow j}(\bm X_{t})dt + o(dt).
\end{align*}
For our algorithm development and Experiment $1$ in Section
\ref{sec:NumericalExperiments}, we take 
\begin{align}
p_{t, i\rightarrow j}(\bm x)
:= p_{t, i\rightarrow j}(x_i)
:= \mathbb P[X_{t,j} = x_i| \mathcal F_t],
\end{align}
but for Experiment $2$ in Section \ref{sec:NumericalExperiments},
we take $p_{t,i\rightarrow j}(\bm x) =
\exp(-\|x_i-x_j\|^2)$.  
Next, by way of assumption, for each pair, say, actor
$i$ and actor $j$, we eliminate the possibility that both actor $i$ and actor
$j$ send messages concurrently to each other.  More specifically, we assume
that
\begin{align}
    &\mathbb P[N_{i j}(t+dt) = N_{i j}(t) + 1 | \mathcal F_t, X_{s,i}, X_{s,j}, s\le t ] \\
    &= 
    (\lambda_i(t) \lambda_j(t)/2)
    (p_{t,i\rightarrow j}(X_{t,i}) + p_{t,j \rightarrow i}(X_{t,j} )) dt 
    + o(dt).
\end{align}
For future reference, we let
\begin{align}
    &\lambda_{t,i\rightarrow j}(x) := (\lambda_{t,i} \lambda_{t,j}/2)
    p_{t,i\rightarrow j}(x),\\
    &\lambda_{t,ij} = \lambda_{t,i}\lambda_{t,j} \langle p_{t,i}, p_{t,j}
    \rangle.
\end{align}

\begin{algorithm}[!ht]
    \caption{Simulating messaging activities during a
    \emph{near-infinitesimally-small} time interval}
    \label{algo:MessagingActivities}
    \begin{algorithmic}[1]
        \Require $t \in \mathbb R_+$, $\Delta t \in \mathbb R_+$  and
        $\{ 
            (T_{ij}(t),\lambda_{ij}(t)) \in \mathbb R_+^2 : 1 \le i < j \le n
        \}$
        \vskip0.1in 
        \Procedure {MessagingActivities}{}
        \State $\ell \leftarrow 1$
            \For {$i \leftarrow 1,\ldots, n-1$}
                \For {$j \leftarrow (i+1),\ldots, n$}
                \State $T_{ij}(t+\Delta t) \leftarrow T_{ij}(t) - \lambda_{t, ij}\Delta t$
                \If {$T_{ij}(t+\Delta t) \le 0$} 
                    \State $\textsc{Messages}[\ell] \leftarrow (t,i,j)$
                    \State $T_{ij}(t+\Delta t) \leftarrow \textsc{UnitExponentialVariable}$
                    \State $\ell \leftarrow \ell + 1$
                    \EndIf
                \EndFor
            \EndFor 
            \State $t \leftarrow t + \Delta t$
        \EndProcedure 
    \end{algorithmic}
\end{algorithm}

\begin{figure}[!ht]
    \centering
    \begin{subfigure}[b]{0.3\textwidth}
        \centering
        {\includegraphics[width=\textwidth]{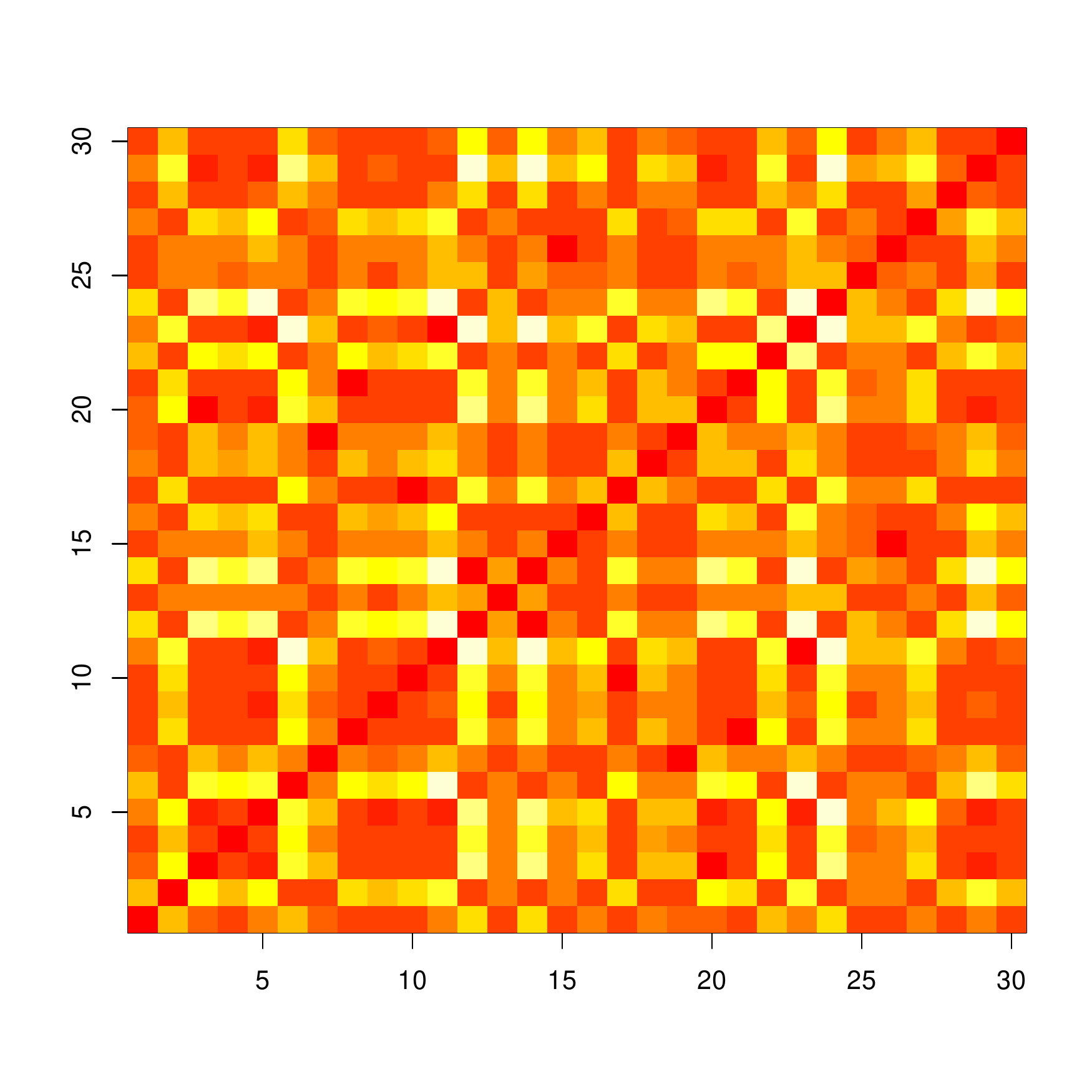}}
        \caption{time $\tau_1$}
    \end{subfigure}
    \begin{subfigure}[b]{0.3\textwidth}
        \centering
        {\includegraphics[width=\textwidth]{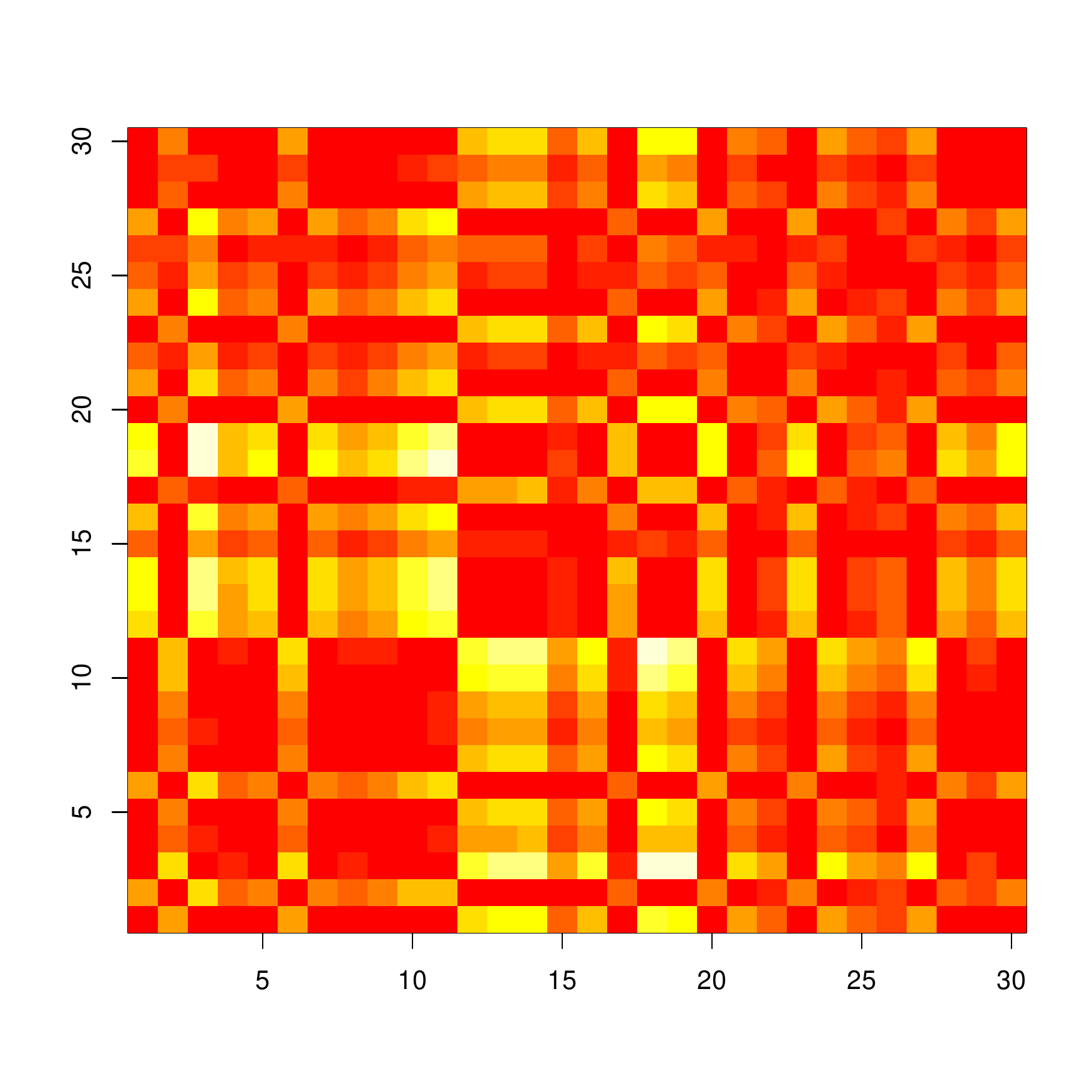}}
        \caption{time $\tau_2$}
    \end{subfigure}
    \begin{subfigure}[b]{0.3\textwidth}
        \centering
        {\includegraphics[width=\textwidth]{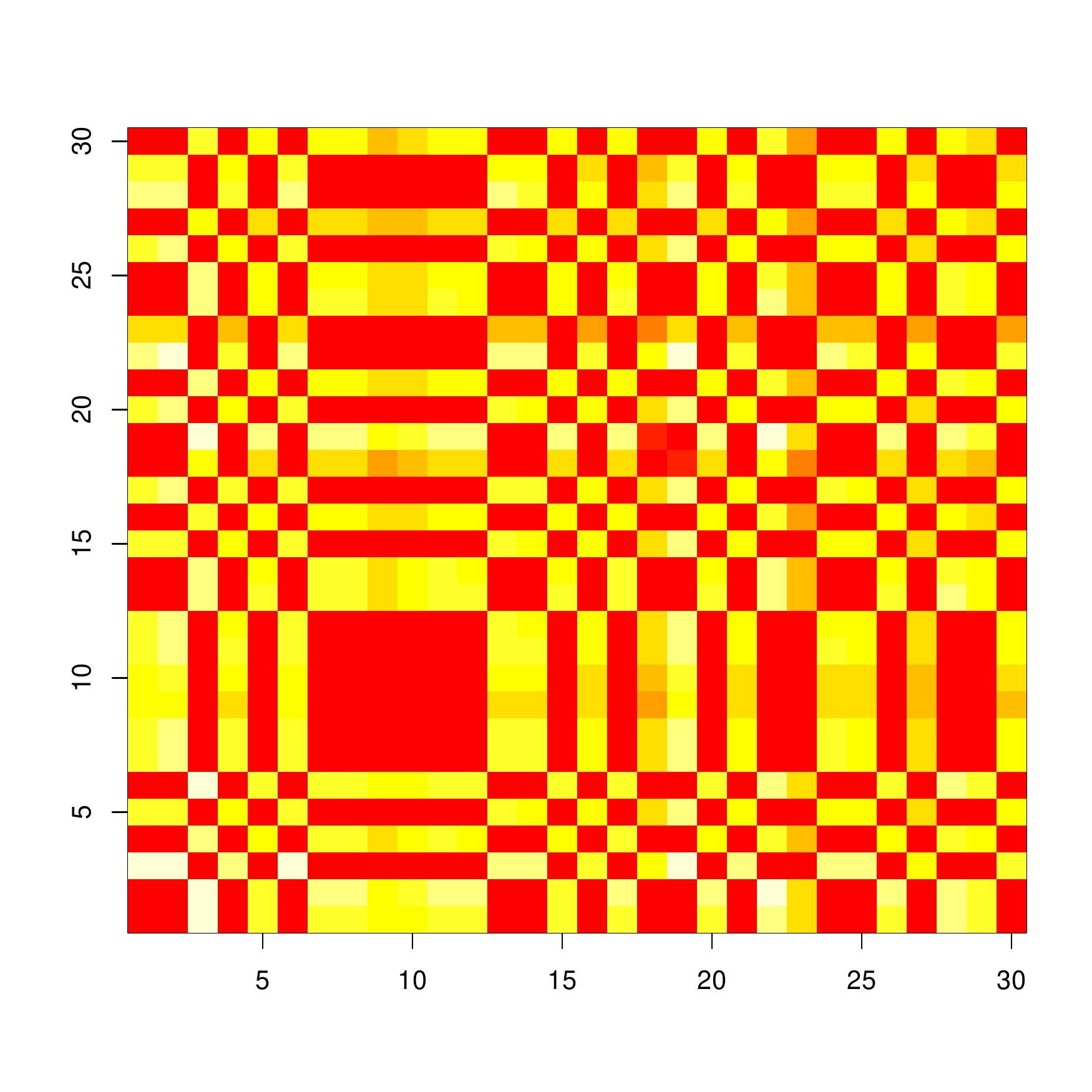}}
        \caption{time $\tau_3$}
    \end{subfigure}
    \caption{One simulation's Kullback-Leibler divergence of posteriors at
    times $\tau_1<\tau_2<\tau_3$.  The horizontal ($x$) and vertical ($y$)
    values, ranging in $1,2,\ldots,30$, label actors.  The more red the cell
    at $(x,y)$ is, the more similar vertex $y$ is to vertex $x$.  The
    dissimilarity measure (KL) clearly indicates the emergence of vertex
    clustering.} \label{fig:Kullback-Leibler divergence - Simulation}
\end{figure}

\section{Algorithm for computing posterior processes} 
\label{sec:SimplifiedApproaches}

We denote by ${\rho}_t$ the conditional distribution of ${\bm X}(t)$
given $\mathcal F_t$, i.e., for each $B \in \mathcal{B}(\mathbb X)$, 
\begin{align}
    {\rho}_t(B) = \mathbb P[{\mathbf X}_t \in B \left|\mathcal F_t\right.].
\end{align}
For the rest of this paper, we shall assume that the (random) measure 
$\rho_{t}(d\bm x)$ is absolutely continuous with respect to Lebesgue measure
with its density denoted by $p_t(\bm x)$.  That is, 
$\rho_t(B) = \int \bm 1_{B}(\bm x) p_t(\bm x) d\bm x$ for $B \in \mathcal
B(\mathbb X)$.
Denote by $\rho_{t,i}$ the $i$-th marginal posterior
distribution of ${\rho}_t$, i.e., for each $B \in \mathcal B(\mathbb R^d)$,
$\rho_{t,i}(B) = \mathbb P[X_{t,i} \in B \left|\mathcal F_t\right.]$,
and let $p_{t,i}(x)$ denote its density. 

\subsection{Theoretical background} \label{sub:Conditionaldistributionprocess}
In Theorem \ref{thm:exactposterior}, 
the \emph{exact} formula for updating the posterior is presented, 
and in Theorem \ref{thm:simplifyingposteriorupdaterule}, our \emph{working}
formula used in our numerical experiments is given. We develop our 
theory for the case where $\omega_t$ and $\sigma_t$ are the same for all
actors for simplicity, as generalization to the case of each actor having
different values for $\omega_{t,i}$ and $\sigma_{t,i}$ is straightforward but
requires some additional notational complexity.

\begin{theorem}\label{thm:exactposterior}
    For each $f \in {C}_b(\mathbb X)$ and $t\in (0,\infty)$, 
\begin{align*} 
    d\rho_t(f) &=
    \rho_t\left(\mathcal{A}_tf\right) dt +
    \bm 1^{\top}
    \left(
        \int_{\mathbb X } 
        \rho_t(d\bm x)
        f(\bm x) \left(\widetilde{\bm \lambda}_t(\bm x) - {\bm 1} {\bm 1}^\top \right)
        * d\bm{M}_t
        \right)
        \bm 1,
\end{align*}
where $\widetilde{\bm \lambda}_t(\bm x) = (\widetilde{\lambda}_{t,ij}(\bm
x))_{i,j=1}^n$ is an $n\times n$ matrix such that for each $i \neq j$, 
$\widetilde{\lambda}_{t,ij}(\bm x) = p_{t, j}(x_i)/
\langle p_{t,i},p_{t,j}\rangle$ and for each $i=j$, $\widetilde{\lambda}_{t,ij}(\bm x) = 1$,
and $d{\bm M}_t=(dM_{t,ij})_{i,j=1}^n$ is an $n\times n$ matrix such that for
each pair $i\neq j$, $dM_{t,ij} = d{N}_{t,ij} - {\lambda}_{t, ij} dt$
and for each pair $i=j$, $dM_{t,ij} = 0$.
\end{theorem}

Hereafter, for developing algorithms further for efficient computations, 
we make the assumption that 
for each $t$, 
\begin{align}
    p_{t,ijk} = p_{t,i} p_{t,j} p_{t,k}, \label{productjointassumption}
\end{align}
where $p_{t,ijk}$ denotes the joint density for actors $i$, $j$ and $k$.

\begin{theorem}\label{thm:simplifyingposteriorupdaterule}
For each function $f \in C_0(\mathbb R^d)$, we have
\begin{align*}
    dp_{t,i}(f) 
    = 
    p_{t,i}(\mathcal A_t f) dt 
    + 
    \sum_{\substack{j\neq i}}
    \left(
    \frac{\langle f, p_{t,i} p_{t,j}\rangle }{\langle p_{t,i},p_{t,j}\rangle} - p_{t,i}(f)
    \right)
    \left( dN_{t,ij} - \lambda_{t,ij} dt \right).
\end{align*}
\end{theorem}
Replacing $f$ with a Dirac delta generalized function, Theorem
\ref{thm:simplifyingposteriorupdaterule} states that for each $x\in \mathbb
R^d$, 
\begin{align*}
    dp_{t,i}(x) 
    = 
    \mathcal A_t^* p_{t,i}(x) dt 
    + 
    p_{t,i}(x)
    \sum_{\substack{j\neq i}}
    \left(
        \frac{p_{t,j}(x)}{\langle p_{t,i},p_{t,j}\rangle} -1
    \right)
    \left( dN_{t,ij} - \lambda_{t,ij} dt \right),
\end{align*}
where $\mathcal A^*$ denotes the formal adjoint operator of $\mathcal A$.  For
use only within Algorithm \ref{algo:EstimateActorPosterior},
\begin{align}
&\textsc{PdeTerm}_{t,i} = \mathcal A_t^* p_{t,i}(x) dt, \text{ and } \label{eqn:PdeTerm}\\
&\textsc{JumpTerm}_{t,i}
= p_{t,i}(x)
\sum_{\substack{j\neq i}}
    \left(
    \frac{p_{t,j}(x)}{\langle p_{t,i},p_{t,j}\rangle} -1
    \right)
    \left( dN_{t,ij} - \lambda_{t,ij} dt \right).\label{eqn:JumpTerm}
\end{align}

\begin{algorithm}
    \caption{Updating the posterior distribution of actors' latent position
    over a \emph{near-infinitesimally-small} time interval}
    \label{algo:EstimateActorPosterior}
    \begin{algorithmic}[1]
 \Require $t \in \mathbb R_+$, $\Delta t \in \mathbb R_+$, $\mu_t$ and
 $\{(t_\ell, u_\ell, v_\ell): t_\ell \in (t,t+\Delta t], 1 \le u_\ell < v_\ell
 \le n\}$
        \vskip0.1in 
        \Procedure {EstimateActorPosterior}{}
            \For {$i \leftarrow 1,\ldots, n$}
            \State Compute $p_{t+\Delta t,i}$ from $p_{t,i}$ using
            $\textsc{PdeTerm}_{t,i}$
            \EndFor 
            \State $\ell \leftarrow \argmin_{m} t_m$
            \State $t_{\ell-1} \leftarrow t$
            \While{$t_\ell \in (t,t+\Delta t]$}
                \State $dt \leftarrow t_{\ell} - t_{\ell-1}$
                \For {$i \leftarrow 1,\ldots, n-1$}
                    \For{$j \leftarrow (i+1),\ldots, n$} 
                \If{$\{i,j\} = \{u_\ell,v_\ell\}$}
                    \State $dN_{t,ij} = 1$
                \Else 
                    \State $dN_{t,ij} = 0$
                \EndIf
                    \State Update $p_{t+\Delta t,i}$ using $(p_{t,m})_{m=1}^n$ and
                    $\textsc{PdeTerm}_{t,i}$ in \eqref{eqn:PdeTerm}
                    \State Update $p_{t+\Delta t,j}$ using $(p_{t,m})_{m=1}^n$ and
                    $\textsc{JumpTerm}_{t,j}$ in \eqref{eqn:JumpTerm}
                    \State $\ell \leftarrow \ell + 1$
                    \EndFor
                \EndFor
            \EndWhile
        \EndProcedure 
    \end{algorithmic}
\end{algorithm}

\subsection{A mixture projection approach}
The projection filter is an algorithm which provides an approximation to the
conditional distribution of the latent process in a systematic way, the method
being based on the differential geometric approach to statistics,
cf.~\cite{BainCrisan}.  When the space on which we project is a mixture
family, as in \cite{BrigoD2011}, the projection filter is equivalent to an
approximate filtering via the Galerkin method,
cf.~\cite{GuntherBeardWilsonOliphantStirling1997}.  Following this idea,
starting from Theorem \ref{thm:simplifyingposteriorupdaterule}, we obtain
below in Theorem \ref{thm:ProjFilterZakai} the basic formula for our
approximate filtering algorithm.

To be more specific, consider a set of probability density functions $\mathcal
S \equiv \{\phi_\ell\}$.  Then, let $\overline{\mathcal S} \le \mathcal
M_1(\mathbb R^d)$ be the space of all probability density functions that can
be written as a probability-weighted sum of $\phi_1,\phi_2, \cdots$. That is,
$f \in \overline{\mathcal S}$ if and only if $f(x) = \sum_{\ell} w_\ell
\phi_\ell(x)$ for some probability vector $(w_1,w_2,\cdots)^\top$ on indices
$1,2,\cdots$.  In particular, for deriving our algorithms, we will assume that
for some systematic choice of $\mathcal S$, each probability density under
consideration is a member of $\overline{\mathcal S}$.  

Among many possible choices for $\{\phi_k\}$ in Theorem
\ref{thm:ProjFilterZakai} are a multivariate Haar wavelet basis and a
multivariate Daubechies basis.  On the other hand, a Gaussian mixture model is
pervasively used throughout statistical inference tasks such as clustering and
classification in algorithms such as $k$-means clustering. As such, we develop
our algorithms with an eye towards use with other Gaussian mixture
model-based algorithms.  In Appendix \ref{appendix:mixtureprojectionformula},
we further develop our algorithm under the assumption that
\begin{align*} 
    p_{t,i}(x) = \sum_{\ell} W_{t,i,\ell} \phi(x;\theta_\ell,s), 
\end{align*} 
where $\phi$ is the standard Gaussian density function defined on $\mathbb
R^d$, $s \in \mathbb R_+$, and the finite sequence $\{\theta_\ell\}\subset
\mathbb R^d$ is to be chosen judiciously prior to implementing the algorithm.  

Preparing for our next result in Theorem \ref{thm:ProjFilterZakai}, we let $P$ be the
symmetric matrix such that $P_{k_1,k_2} = \langle \phi_{k_1}, \phi_{k_2}
\rangle$, and for each $k$, let $S_k$ be the symmetric matrix such that its
$(r,c)$-entry $S_{k,rc}$ is $\langle \phi_k, \phi_{r}\phi_{c} \rangle$. Collectively, we
denote by $S$ the three-way tensor whose $(k,r,c)$ entry is $S_{k,rc}$.  Let
$R_{t,i,\ell}$ be the matrix such that its $(r,c)$-entry $R_{t,i,\ell,rc}$ is
$\langle \mathcal A_{t,i,\ell} \phi_r, \phi_c \rangle$, where
$\mathcal A_{t,i,\ell}$ is the differential operator such that
\begin{align}
    \mathcal A_{t,i,\ell} f(x) 
    = 
    \sum_{k=1}^d b_{t,i,\ell}^{k}( x) \frac{\partial }{\partial x_k} f( x) 
    +
    \sum_{k_1,k_2=1}^d a_{t,i,\ell}^{k_1,k_2}( x) \frac{\partial^2}{\partial
    x_{k_1} \partial x_{k_2}} f( x),
    \label{eqn:AtiellVSRtiellrc}
\end{align} 
with
\begin{align}
    &b_{t,i,\ell}^{k}( x) 
    = 2(1-\omega_{t,i})\int \psi\left(\sigma_{t,i}^{-1} ( y- x)\right) (y - x)_k
    \phi(y;c_\ell, \alpha_{t,\ell}) d y, \label{eqn:btiellk2} \\
    &a_{t,i,\ell}^{k_1,k_2}( x) 
    = (1-\omega_{t,i})^2 \int \psi\left(\sigma_{t,i}^{-1} ( y- x) \right) (y-
    x)_{k_1} (y- x)_{k_2} \phi(y;c_\ell,\alpha_{t,\ell}) d y.
    \label{atiellk1k2}
\end{align}

\begin{theorem}\label{thm:ProjFilterZakai}
Suppose that for each $t$, $i$ and $x$, 
$p_{t,i}(x) = \sum_{k=1}^K W_{t,i,k} \phi_k(x)$.
Let $W_{t,i}$ denote the column vector whose $k$-th 
entry is $W_{t,i,k}$. Then, 
\begin{align}
    P d{W}_{t,i}
    &= 
    \sum_{\ell} q_{t,\ell} R_{t,i,\ell} W_{t,i} dt \nonumber \\
    &\quad +
    \sum_{j\neq i } 
    \left(
    \dfrac{(W_{t,i}^\top S_r W_{t,j})_{r=1}^K}{W_{t,i}^\top P W_{t,j}}
    - 
    P W_{t,i}
    \right)
    \left(
    dN_{t,ij} 
    - 
    \lambda_{t,i}\lambda_{t,j} W_{t,i}^\top P W_{t,j} dt 
    \right).\label{eqn:ProjFilterZakai}
\end{align}
\end{theorem}

\subsection{Algorithm for continuous embeddings}\label{sub:Continuousembedding}

\subsubsection{Classical multidimensional scaling} 
In our application, our final analysis is completed by clustering the
posterior distributions.  Instead of working directly with posteriors, an
infinite-dimensional object, we propose to work with objects in an Euclidean
space each of which represents a particular actor.  However, given $p_{t,i}$
and $p_{t,j}$, using their mean vectors or their KL distance for clustering
can be uninformative.  For example, if $p_{t,i} = p_{t,j}$, then their mean
vectors would be the same and their KL distance  would be zero. However, if
$p_{t,i} = p_{t,j}$ is the density of, say, a normal random vector such that
its mean is zero and its covariance matrix is $v I$ for a large value of $v$,
then concluding that actor $i$ and actor $j$ are similar could be misleading.

To alleviate such situations in a clustering step of our numerical
experiments,  we propose using a multivariate statistical technique called 
classical multidimensional scaling (CMDS) to obtain a lower dimensional
representation of $\bm p_{t} = (p_{t,i})_{i=1}^n$.  More specifically, we
achieve this by representing each actor as a point in $\mathbb R^d$, where the
configuration is obtained by solving the optimization problem
\begin{align}
    \min_{x_1,\ldots, x_n \in \mathbb R^d} \sum_{i < j} 
    \big| 
    \|x_i - x_j\| - g(\langle p_{t,i},p_{t,j}\rangle))
    \big|^2,\label{eqn:optsoln}
\end{align}
where $g$ denotes a strictly decreasing function defined on $[0,\infty)$
taking values in $\mathbb R_+$. 
For example, one can take $g(u) =
\cos^{-1}(\omega u)$ where $\omega \in \mathbb R_+$ is chosen so that 
$\omega u \in [0,\pi/2]$ for all possible values $u$ of $\langle p_{t,i}, p_{t,j} \rangle$. 
Another possibility among many others is to choose $g(u) = -\log(\omega u)$
if $\omega$ is chosen so that $\omega u \in (0,1]$ for all possible values
$u$ of $\langle p_{t,i},p_{t,j}\rangle$. 

We denote by $\xi(\bm p)$ the set of solutions to the optimization
problem \eqref{eqn:optsoln}.  Given a vector $\bm p$ of $n$ probability
densities on $\mathbb R^d$, it can be shown that the solution set $\xi(\bm p)$ 
is not empty and is closed under orthogonal transformations.

In the classical embedding literature, ensuring continuous embeddings is
neglected as it is not relevant to their applications.  However, for our work,
this is crucial as we study their evolution through time, i.e., ideally, we
would like to see that a small change in time corresponds to a small change in
latent location.  In this section, we propose an extension to the CMDS
algorithm to remedy the aforementioned non-uniqueness issue, and show that the
resulting algorithm ensures continuity of embeddings.

In our numerical experiments, for each $\bm p$, we choose a particular element
$\xi^*(\bm p)$ of the solution set $\xi(\bm p)$ so that $\xi^*(\bm p)$ depends
on $\bm p$ in a consistent manner.

\begin{algorithm}
	\caption{Compute a unique CMDS embedding of $M$ by minimizing the distance from 
    a reference configuration $Z$ with full column rank}
    \label{algo:CMDScc}
\begin{algorithmic}[1]
	\Require a matrix $Z \in \mathbb{M}_{n \times d}$ with full column rank
    and a symmetric matrix $M \in \mathbb{M}_{n \times n}^+$ such that $\diag(M) =0$
    \State $B$ $\leftarrow$ any $n \times d$ classical MDS solution of $M$
    \State Compute a singular value decomposition $U D V^\top$ of $Z^\top B$  
    \State Return $BVU^\top$
\end{algorithmic}
\end{algorithm}

\subsubsection{Continuous selection}

By a dissimilarity matrix, we shall mean a real symmetric non-negative matrix
whose diagonal entries are all zeros.  First, fix $d$ such that $1 \le d \le
n$.  Then, for each $n\times n$ dissimilarity matrix $M$, define
\begin{align*}
&\varrho(M) = -\frac{1}{2}({\bf I} - {\bf 1} {\bf 1}^\top/n) M^{(2)} ({\bf I}
- {\bf 1} {\bf 1}^\top/n),\\
&\xi_d^\dagger(M) = \argmin_{X\in\mathbb R^{n\times d}} \| \varrho(M) - X
X^\top\|_F^2,
\end{align*}
where $M^{(2)} = (M_{ij}^{2})$.  The elements of $\xi_d^\dagger(M)$ are known
as \emph{classical multidimensional scalings}, and as discussed in
\citet{BorgGroenen2005}, it is well known that $\xi_d^\dagger(M)$ is not
empty provided that the rank of $\varrho(M)$ is at least $d$.  Our discussion
in this section concerns making a selection $\xi_d^*(M)$ from
$\xi_d^\dagger(M)$ so that the map $M\rightarrow \xi_d^*(M)$ is continuous
over the set of dissimilarity matrices such that $\varrho(M)$ is of rank at
least $d$. 

Let $M$ be a dissimilarity matrix such that the rank of $\varrho(M)$ is at
least $d$.  We begin by choosing an element of $\xi_d^\dagger(M)$, say
$\xi_d(M)$, through classical dimensional scaling.  Let $U\Sigma U^\top$ be
the eigenvalue decomposition of $\xi_d(M)$, where $UU^\top= I$ and $\Sigma$ is
the diagonal matrix whose entries are the eigenvalues in non-increasing order.
By the rank condition, we have $\Sigma_{11} \ge \ldots \ge \Sigma_{dd} > 0$.
First, we formally write 
\begin{align}
    X_+ = U_+\sqrt{\Sigma_+}, 
\end{align}
where 
\begin{description}
    \item[(i)] $U_+$ is the $n\times d$ matrix with its $ij$ entry  $U_{ij}$, 
    \item[(ii)] $\Sigma_+$ is the $d\times d$ diagonal matrix whose $i$-th
        diagonal entry is $\Sigma_{ii}$.
\end{description}
Dependence of $X_+$ on $M$ will be suppressed in our notation unless needed
for clarity.  Now, if the diagonals of $\Sigma_+$ are distinct, then $X_+$
is well defined.  However, in general, due to potential geometric multiplicity
of an eigenvalue, our definition of $X_+$ can be ambiguous.  This is the main
challenge in making a continuous selection and we resolve this issue in our
following discussion. 

For our remaining discussion, without loss of generality, we
may assume that for each dissimilarity matrix $M$, $X_+$ is well-defined
by making an arbitrary choice if there is more than one CMDS solution.  Note
that the mapping $M\rightarrow X_+(M)$ may not be a continuous selection.  We
now remedy this.  First, fix an $n\times d$ matrix $Z$ and let
\begin{align*}
    \xi_d(M) = \left\{ X_+Q: QQ^\top = I \right\} \subset \xi_d^\dagger(M),
\end{align*}
where $Q$ runs over all $d\times d$ real orthogonal matrices. 
Then, define
\begin{align}
\xi_d^*(M) \equiv \argmin_{X \in \xi_d(M)} \|X -  Z\|_F^2.
\end{align}

Algorithm \ref{algo:CMDScc} yields the solution $\xi_d^*(M)$ and the proof of the following
theorem, Theorem \ref{thmstat:continuousembedding}, can be found in Appendix
\ref{thmstat:contembedding:proof}.

\begin{theorem}\label{thmstat:continuousembedding}
Suppose that the $n\times d$ matrix $Z$ is of full column rank.  Then, the
mapping $M \rightarrow \xi_d^*(M)$ yields a well-defined continuous function
on the space of dissimilarity matrices such that $\varrho(M)$ is of rank at
least $d$. 
\end{theorem}

\begin{figure}[!ht]
        \centering
        \includegraphics[height=0.4\textheight]{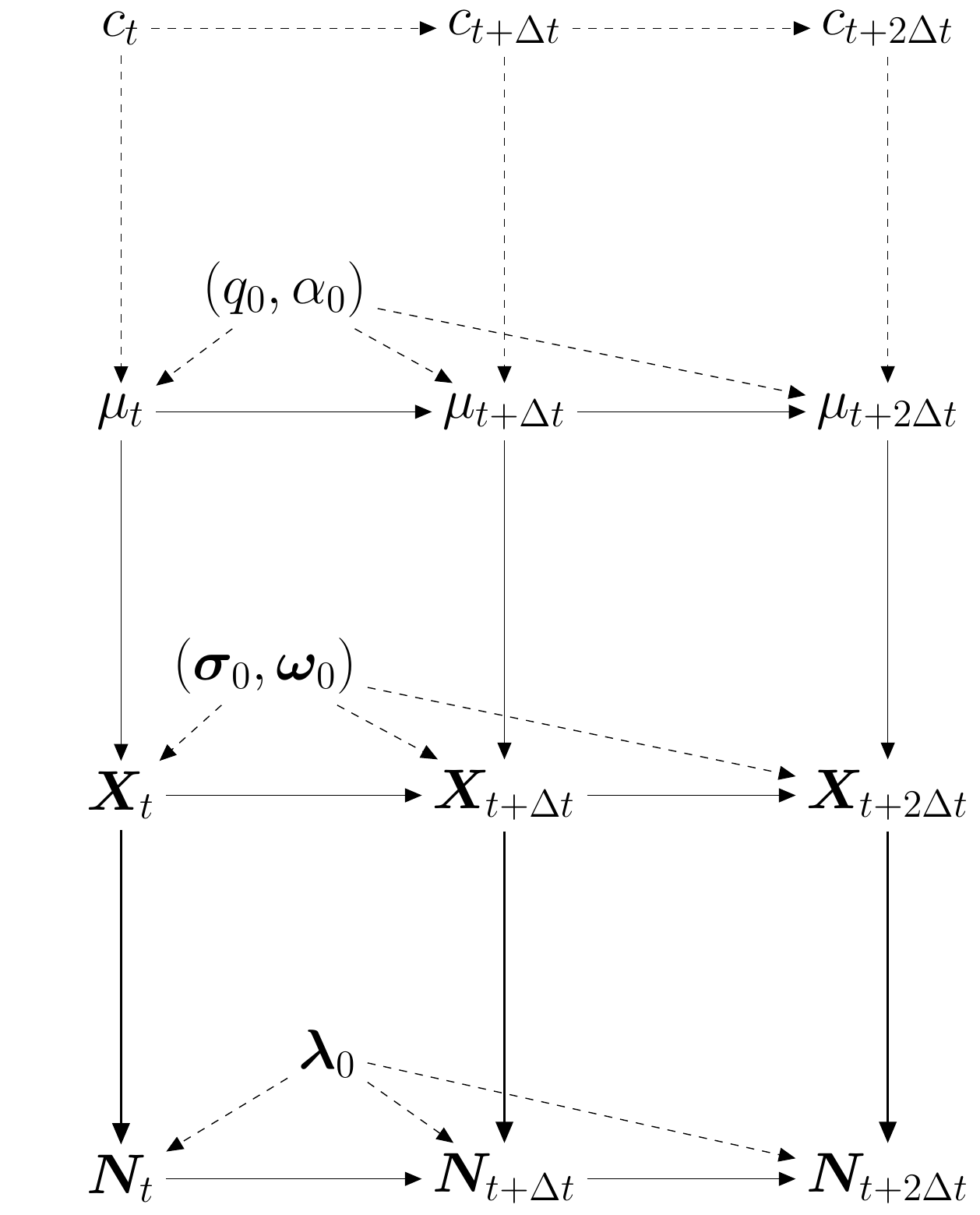}
    \caption{A graphical representation of the dependence structure in the
        simulation experiment.  The nodes that originate the dashed lines 
        correspond to either one of constant model parameters ($(\bm q_0, \bm \alpha_0)$,
        $(\bm \sigma_0, \bm \omega_0)$, and $\bm \lambda_0$) or exogenous modeling element
        $(c_t)$. The arrows are associated with \emph{influence} relationships, e.g, $\mu_t
        \rightarrow \mu_{t+\Delta t}$ reads $\mu_t$ influences $\mu_{t+\Delta
        t}$. 
    }
    \label{fig:hierarchalmodeldiagramwithparam}
\end{figure}

\subsection{Technical observations} 
Here we discuss some insightful facts related to our model given 
the assumption stated in the last section. First, we have the following:
\begin{theorem}\label{thm:elliptic}
Fix $t\ge 0$ and  suppose that $\mu_t(x) > 0$ for a.e.~$x \in \mathbb R^d$. 
The operator $\mathcal A_t$ is elliptic, i.e., for each $x \in
\mathbb R^d$, the matrix $a_t(x)=(a_t^{k_1,k_2}(x))$ is positive definite.
\end{theorem}
\begin{proof}
    Note that for each $z \in \mathbb R^d$, 
    \begin{align*}
        z^\top a_t(x) z 
        &= 
        (1-\omega_t)^2 
        \int 
            \psi(\sigma_t^{-1}(y-x)) 
            \left(
            z^\top
            (y-x)
            (y-x)^\top 
            z
            \right)
            \mu_t(y)
            dy\\
        &= 
        (1-\omega_t)^2 
        \int 
            \psi(\sigma_t^{-1}(y-x)) 
            \left|(y-x)^\top z\right|^2
            \mu_t(y)
            dy.
    \end{align*}
    Note that $\psi(\sigma_t^{-1}(y-x))>0 $ and $\mu_t(y) > 0$ for each $y
    \in \mathbb R^d$, and that $|(y-x)^\top z|^2 > 0$ away from a subspace of
    $\mathbb R^d$ whose Lebesgue measure is zero. It follows that 
    $z^\top a_t(x) z > 0$ for each $x\in \mathbb R^d$, whence  $a_t(x)$ is
    positive definite.  
\end{proof}

Now, we further examine Algorithm \ref{algo:PopulationProcess}, Algorithm
\ref{algo:LatentProcess}, Algorithm \ref{algo:MessagingActivities}, Algorithm
\ref{algo:EstimateActorPosterior}, Algorithm \ref{algo:CMDScc}, and discuss
some technical points behinds these algorithms.

In Algorithm \ref{algo:LatentProcess}, existence and uniqueness of
$\sqrt{a_t(x)}$ follows from Theorem \ref{thm:elliptic}.  The continuity of
$x\rightarrow \sqrt{a_t(x)}$ follows from Theorem \ref{thm:elliptic} and
\cite{StrookPDE2008}.  In Algorithm \ref{algo:MessagingActivities}, for
simulating a sample path of $t \rightarrow N_{t, ij}$, we use the so-called
time-change property; that is, we use the fact that the process given by $t
\rightarrow N_{t,ij}^*$ is a unit-rate simple Poisson process, where
$\Lambda_{t,ij} = \int_0^t \lambda_{u,ij}du$ and $\Lambda_{t,ij}^{-1} :=
\inf\{ u \ge 0: \Lambda_{u,ij} \ge t \}$ and $N_{t,ij}^* := \left.
N_{u,ij}\right|_{u=\Lambda_{t,ij}^{-1}}$.  For simulation, we use its dual
result, i.e., $t\rightarrow N_{t,ij} := \left.
N_{u,ij}^*\right|_{u=\Lambda_{t,ij}}$ is a point process whose intensity 
process is $\lambda_{t,ij}$, where $u\rightarrow N_{u,ij}^*$ is a
path of a unit-rate simple Poisson process.  Also, we note that for
computation of $\lambda_{t,ij}$, online inference is a necessary part of our
simulation in Algorithm \ref{algo:MessagingActivities}; that is, we need to
compute $p_{t,i\rightarrow j}(x) = \mathbb P[X_{t,j}=x\left|\mathcal
F_t\right.].$ In Algorithm \ref{algo:MessagingActivities} and Algorithm
\ref{algo:EstimateActorPosterior}, \emph{near-infinitesimally-small} means
$\Delta t$ so small that the likelihood of having more than one event during a
time interval of length $\Delta t$ is practically negligible.  Also, by
\textsc{StandardNormalVector} in Algorithm \ref{algo:LatentProcess} and
\textsc{UnitExponentialVariable} in Algorithm \ref{algo:MessagingActivities},
we mean generating, respectively,  a single normal random vector with its mean
vector being zero and its covariance matrix being the identity matrix, and a
single exponential random variable whose mean is one.  

\section{Simulation experiments} \label{sec:NumericalExperiments}

In our experiments, we hope to detect clusters with accuracy and speed similar
to that possible if the latent positions $\bm X(t)$ were actually observed
even though we use only information in $\bm p_t = (p_{t,i})_{i=1}^n$ estimated
from information contained in $\mathcal F_t$.  We denote the end-time for our
simulation as $T$.  There are two simulation experiments presented in
this section, and the computing environment used in each experiment is
reported at the end of this section.

\paragraph{Experiment 1} 
We take $d=1$ and we assume that for each $t \in [0,T]$ and actor $i=1,\ldots,
8$, $\lambda_{t,i}=5, \sigma_{t,i}^2={1/3}$ and $\omega_{t,i}=0.1$.
For the population process we take for each $t \in [0,T]$
\begin{align*}
    \mu_{t,\texttt{I}}(x) =  \phi(x;c_t,\alpha_t),
\end{align*}
where
\begin{align*}
    &\alpha_{t} = 1/3, \text{ and } \\
    &c_{t,\texttt{I}} \equiv
    \begin{cases} 
        1 & \mbox{if } t \in [0,100\Delta t), \\
        0.5 & \mbox{if } t \in [100\Delta t, 250\Delta t),\\ 
        0 & \mbox{if } t \in [250 \Delta t, 500\Delta t].
    \end{cases} 
\end{align*}
Then, we also consider $\mu_{t,\texttt{II}}$, where
\begin{align*}
    &\alpha_{t} = 1/3, \text{ and } \\
    &c_{t,\texttt{II}} \equiv
    \begin{cases} 
        1 & \mbox{if } t \in [0,100\Delta t), \\
        0 & \mbox{if } t \in [100\Delta t, 250\Delta t),\\ 
        1 & \mbox{if } t \in [250 \Delta t, 500\Delta t].
    \end{cases} 
\end{align*}

There is only one population density; in other words, $q_{t,i} = 1$.  Note
that even with \emph{one population center}, we can have \emph{more than one
empirical mode} for the subpopulation.  One of these modes is near zero, and 
another mode is near one.  The reason for this is that because of the
value of $\alpha_t^2=1/3$ and $\sigma_t^2=1/3$, when an actor is too far away
from the mode $c_t$ of the population process, the population process affects
the actors on its \emph{tail} only by negligibly small amount.  In Figure
\ref{fig:experimentone-sixactorpath} and Figure
\ref{fig:experimentone-perfect-filtering-case}, a sample path of the true
latent position of each of eight actors is illustrated in black lines. It is
apparent that in the $c_{t,\texttt{I}}$, all eight actors are equally
\emph{informed} of the population mode shift, but in the $c_{t,\texttt{II}}$
case, only the last three were able to adapt to the change, and the first five
actors are surprised by the abrupt change at time $100\Delta t$. 

Our simulation is discretized. Our unit time is $\Delta t = 0.05$, and in
Figure \ref{fig:experimentone-sixactorpath}, each tick in the horizontal axis
corresponds to an integral multiple of $\Delta t$.  The jump term in our 
update formula is quite sensitive to the number of actors being considered. 
As such, for updating the jump term, we further discretized $\Delta t$ into 
$n^2$ subintervals for numerical stability of our update iterations.  For 
$n=8$, each unit interval is associated with $64$ sub-iterations, and the total
number of the (main) iteration is $400$, and we use $(N_{t+\Delta t,ij} -
N_{t,ij})/n^2$ instead of $dN_{t+\Delta t\ell/n^2,ij} - dN_{t+\Delta
t(\ell-1)/n^2,ij}$ in each $\ell$-th subiteration of each main iteration
staring at time $t$. 

To implement our mixture projection algorithm, we take $s^2=1/2^{12} =
1/4096$.  The initial position of the $n=8$ actors are sampled from the
initial population distribution $\mu_0$.  We take $p_{0,i}(x) =
\phi(x;X_{0,i},s)$.  The discretized version $R_t$ of $\mathcal A_t$ is
illustrated in Figure \ref{fig:levelplot(At)example}. For inference during our
experiment, we have dropped the second order term and used only the first
order term to keep the cost of running our experiment low.
On the other hand, for simulating the actors' latent positions, we
have used both the first and second order term of $\mathcal A_t$.  The value
of $P^{-1} R_t W_{t,i} dt$ gives the first part of the change in $dW_{t,i}$.
Note that in both Figure \ref{fig:levelplot(At)one} and Figure
\ref{fig:levelplot(At)zero}, the entries that are \emph{sufficiently far off}
from the diagonals are near zero.

For $c_t = c_{t,\texttt{I}}$, the time plot of the number of messages produced
during interval $[\Delta t \ell, \Delta t (\ell+1)]$ is given in Figure
\ref{fig:experimentone-num-msg}, and shows transient behaviors of varying
degrees of messaging intensity over the interval.  Our set up for $c_t =
c_{t,\texttt{II}}$ produced a simulation sample output of observing $2.5$
messages amongst the $n=8$ actors in unit time once the population center
changed \emph{abruptly} from $c_t = 1$ to $c_t = 0$ at the start of the
$100$-th unit time interval, i.e., $t=5$.  In other words, after $t \ge 5$, a
single unit time is roughly associated with the amount of time during which
the whole subpopulation of eight actors exchanges around $45$ messages, or
equivalently, during which each pair of actors exchange around $3$ messages.
On the other hand, in both $c_{t,\texttt{I}}$ and $c_{t,\texttt{II}}$ for the
interval $[0,\Delta t 100]=[0,5]$, the subpopulation messaging rate is
relatively constant at the rate of $100$ messages over each unit interval, and
this is expected as all eight actors are tightly situated around $1$. 
\begin{figure}[t]
    \centering
    \begin{subfigure}[b]{0.45\textwidth}
        \centering
        \includegraphics[width=\textwidth]{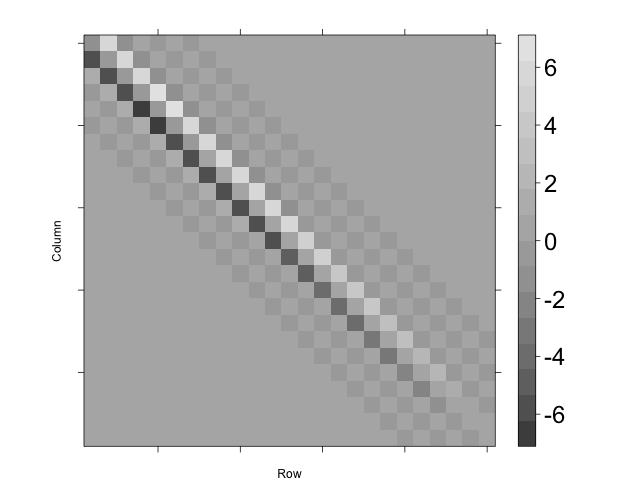}
        \caption{The population mean is at $1$}
        \label{fig:levelplot(At)one}
    \end{subfigure}
    \begin{subfigure}[b]{0.45\textwidth}
        \centering
        \includegraphics[width=\textwidth]{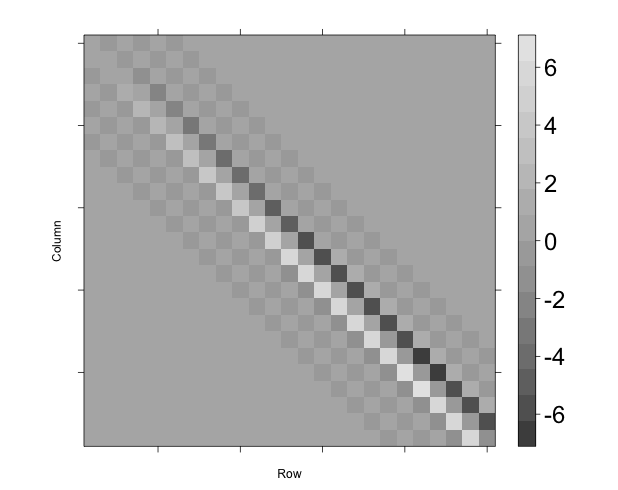}
        \caption{The population mean is at $0$}
        \label{fig:levelplot(At)zero}
    \end{subfigure}
    \caption{The level plots of $P^{-1} R_t$ for the discretized version $R_t$
    of $\mathcal A_t$, used in the simulation experiment for two particular cases, 
    where the horizontal axis is associated with the rows of $P^{-1} R_t$ and the vertical
    axis is associated with the columns of $P^{-1} R_t$.}
    \label{fig:levelplot(At)example}
\end{figure}

\begin{figure}[t]
    \begin{center}
        \includegraphics[width=\textwidth]{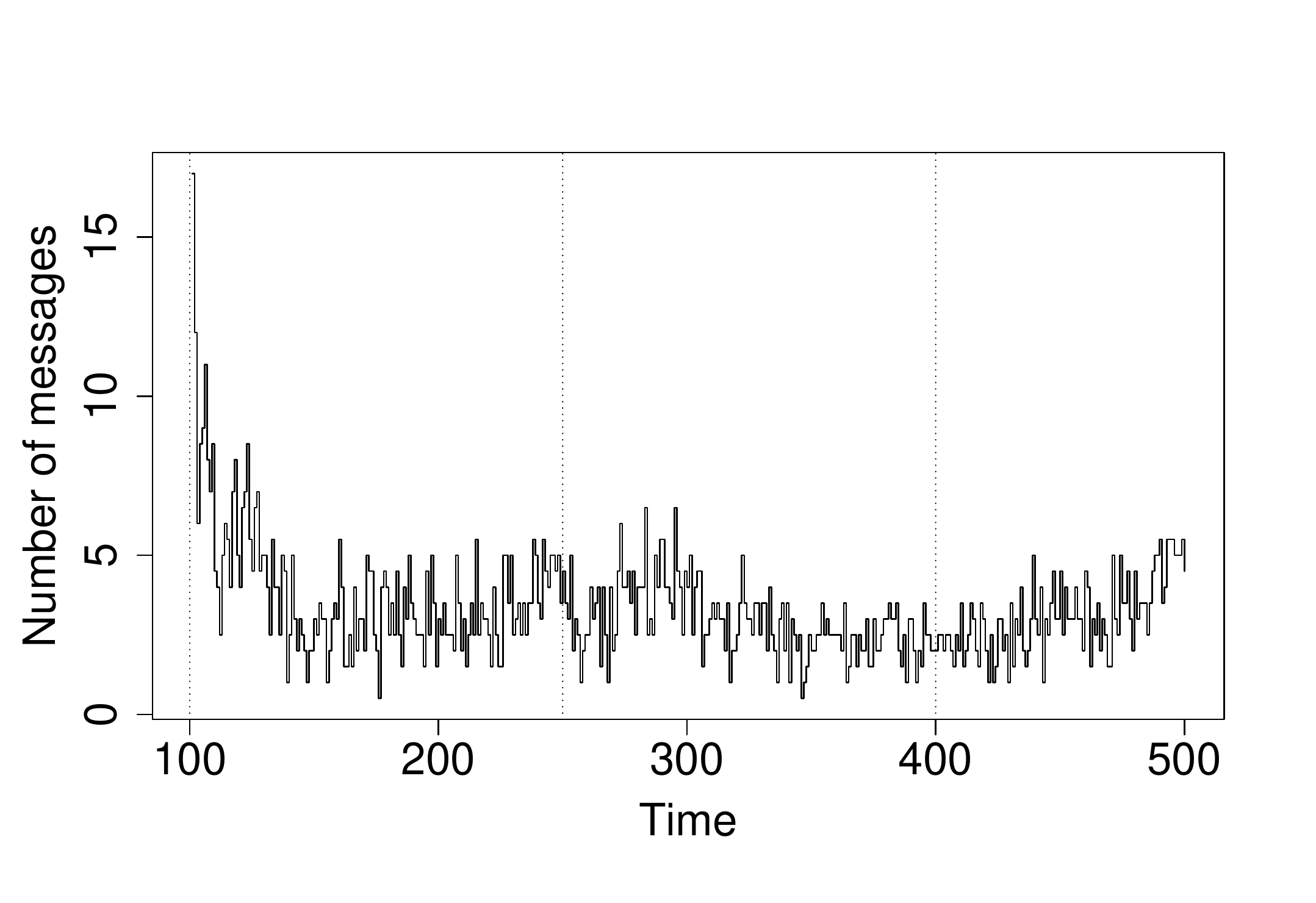}
    \end{center}
    \caption{ The $c_t = c_{t,\texttt{I}}$ case.
    The number of messages per $\Delta t$ across the time interval $[100
    \Delta t, 400 \Delta t]$ for the subpopulation of actors $1,\ldots,8$.}
    \label{fig:experimentone-num-msg}
\end{figure}

\begin{figure}[t]
    \begin{center}
        \includegraphics[width=\textwidth]{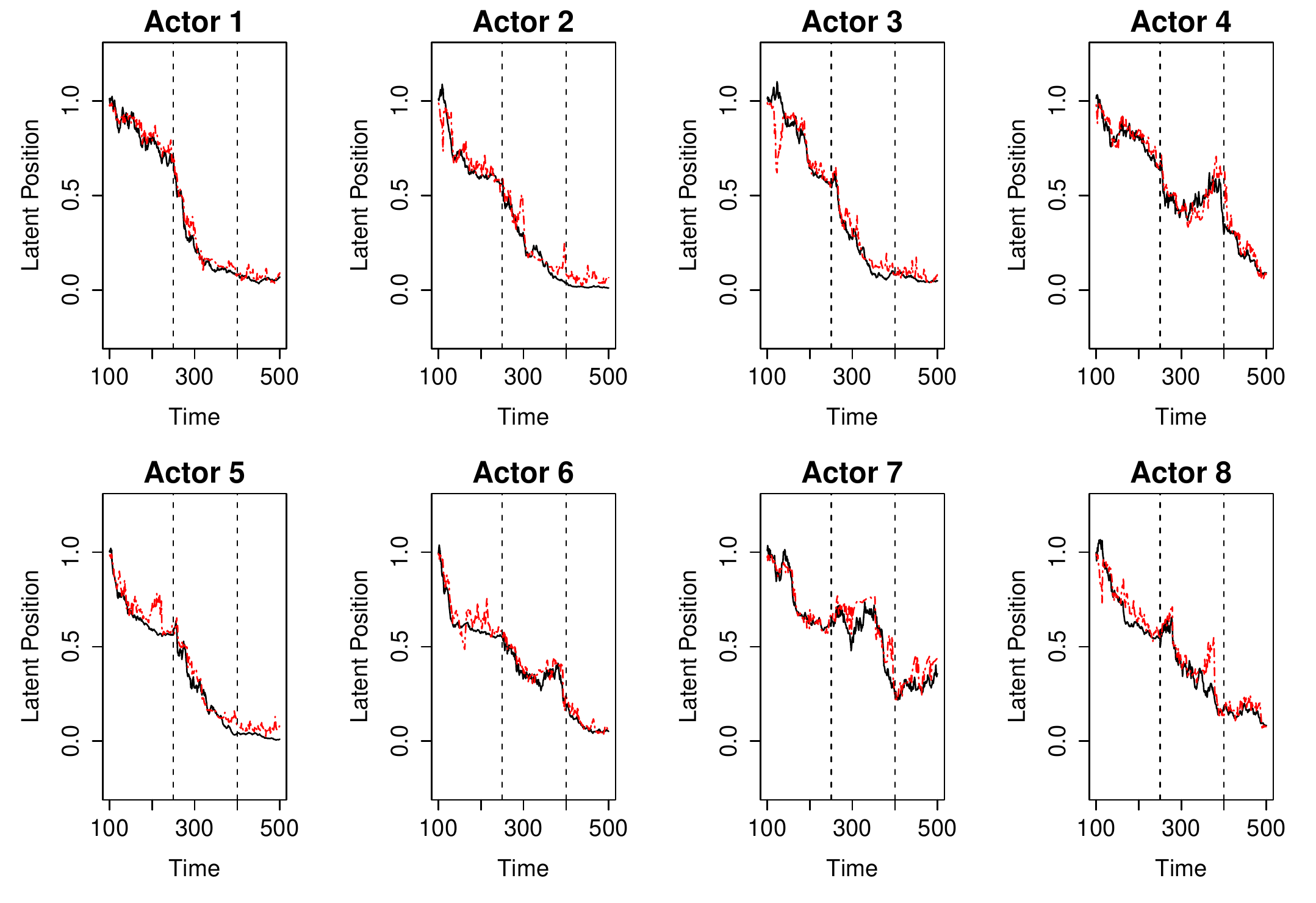}
    \end{center}
    \caption{ The $c_t = c_{t,\texttt{I}}$ case.
    The \emph{sample} path of the true and estimated latent position of
    each of eight actors used for Experiment 1 in a black solid
    line and in a dashed red line, respectively.}
    \label{fig:experimentone-perfect-filtering-case}
\end{figure}

\begin{figure}[t]
    \begin{center}
        \includegraphics[width=\textwidth]{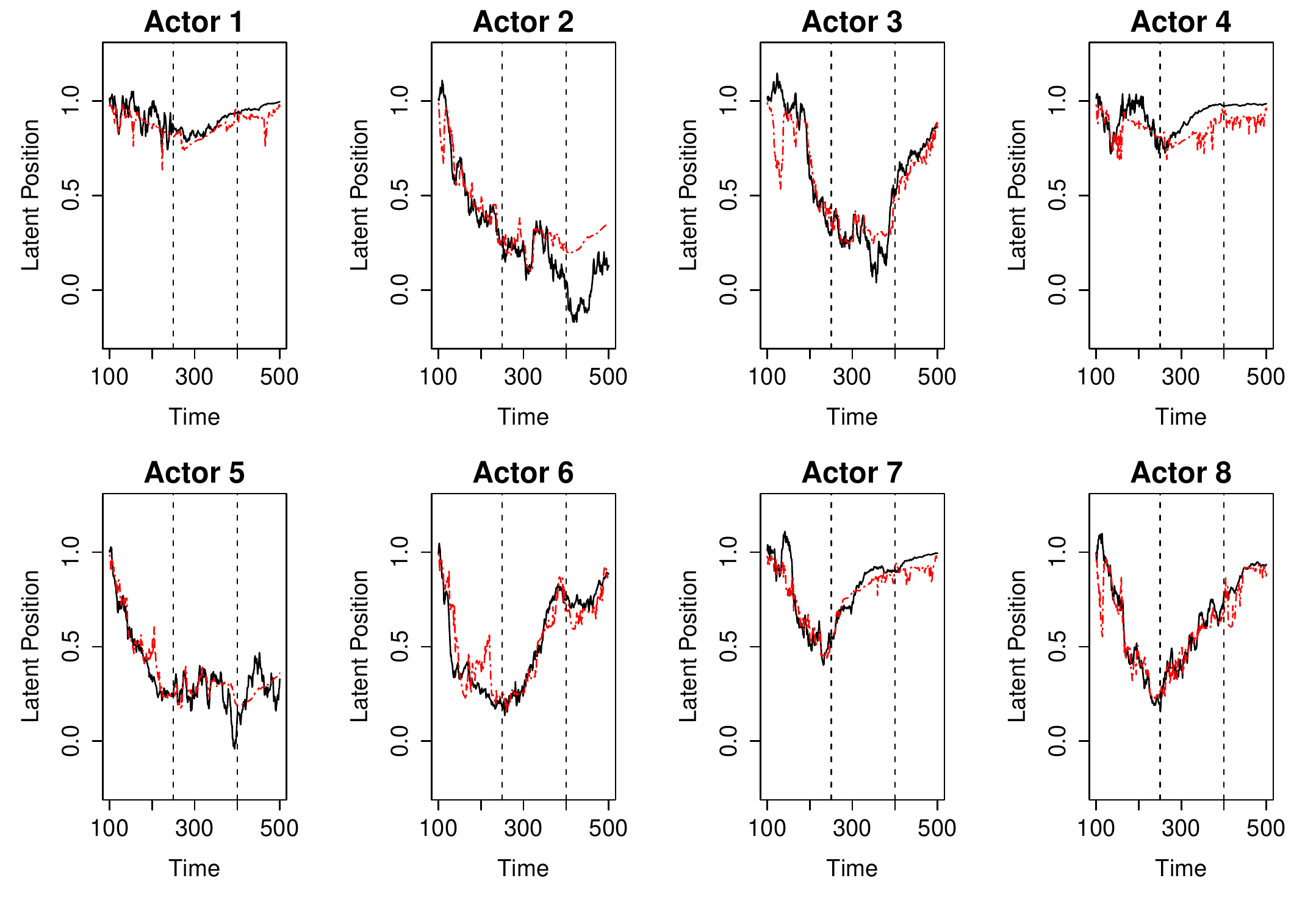}
    \end{center}
    \caption{The $c_t = c_{t,\texttt{II}}$ case.
    The \emph{sample} path of the true and estimated latent position of
    each of eight actors used for Experiment 1 respectively in a black solid
    line and in a dashed red line.}
    \label{fig:experimentone-sixactorpath}
\end{figure}

Our experiments for $c_{t,\texttt{I}}$ and $c_{t,\texttt{II}}$ 
both show that the filtered positions for
all eight actors are close to the exact positions.  \\

\paragraph{Experiment 2}
In this experiment, for each $t$, we have used the empirical distribution of
\begin{align*}
X_{t,n+1},\ldots, X_{t,n+L}
\end{align*}
to obtain an estimate $\widehat{\mu}_t$ of $\mu_t$ by partitioning the latent
space into sufficiently small intervals, where we place a uniform kernel of
height equal to the proportion of $\{X_{n+i} : i= 1,\ldots, L\}$ that lies in
that interval.  Our inference is on $X_{t,1},\ldots, X_{t,n}$.  Recall that
$n$ denotes the size of the subpopulation.  The number $n+L$ is the size of
the \emph{full} population.  This set-up is closer to the motivation for our
work, the bounded confidence model, \citet{GomezGrahamLeBoudec}, and the
connection with our model in this paper is made in Appendix
\ref{appendix:motivationforA}.  In theory, the general setup in Experiment 1
is comparable to the setup in Experiment 2 when $L$ in Experiment 2 is taken
to be $\infty$. 

We set $L=70$, $n=30$, $\Delta=.25$, and $\omega=.2$.  
We take the clustering based on $X_T$ as the \emph{ground truth}.  
Note that $\Delta$ here
is comparable to $\sigma_{t,i}$ in Experiment 1, or more generally, in our
model.  We set up the simulation to
observe roughly 3000 messages amongst the $n$ actors in unit time.
This translates to $10$ per actor per unit time. Note that this is a 
rough estimate as the messaging intensity is time-dependent and stochastic. 
In Figure \ref{fig:MDSvsExact}, we have snapshots of $\bm X_t = (
X_{t,1},\ldots, X_{t,n})$
and those of $\widehat{\bm X}_t = (\widehat{X}_{t,1},\ldots, \widehat{X}_{t,n})$
for a single simulation run.  
Denote as the latency 
\begin{align*}
    \Delta\zeta \equiv \widehat{\zeta} - \zeta,
\end{align*}
where the dependency on our choice for a clustering algorithm is suppressed in
our notation and for some $\varepsilon \in (0,1)$, 
\begin{eqnarray*}
    &\zeta \equiv \inf\{ t \in [0,T] :  \mari(\kappa({\bm X}_s), \kappa({\bm X}_T)) \ge 
    1-\varepsilon, \text{ for a.e.~} s \in [t,T]\},\\
    &\widehat{\zeta} \equiv \inf\{ t \in [0,T] : \mari(\kappa(\psi^*(\mathbf p_s)),
    \kappa({\bm X}_T)) \ge 1-\varepsilon, \text{ for a.e.~} s \in [t,T]\},
\end{eqnarray*}
where $\mari$ denotes a moving average of the Adjusted Rand Index
(c.f.~\cite{WMRand1971} and \cite{HubertArabie1985}) and 
we fix $\kappa$ to be a $k$-means clustering algorithm for concreteness. 
We use the latency as a performance measure for a clustering algorithm $\kappa$
under our framework.  For our projection, we use a Haar basis, i.e., a set of
simple step functions, where the width of the intervals used in the experiment
is $\frac{1}{42}$.  Also, unlike in Experiment 1, we take 
\begin{align*}
    p_{t,i\rightarrow j}(X_{t,i}) 
    = 
    \exp(-\|X_{t,i}-X_{t,j}\|^2). 
\end{align*}
These changes require us to modify our algorithm slightly. However, the
necessary modifications are straightforward, and we leave the details to the
reader. 

It is important to note that we do not assume knowledge of the latent position
of any individual, $X_i(t)$; instead, we use only our knowledge of the overall
population.
As the number $L$ gets larger, as shown in \cite{GomezGrahamLeBoudec},
the dependence among
\begin{align*}
    X_{t,1},\ldots, X_{t,n+L}
\end{align*}
diminishes, agreeing more closely with the model we specified in our
framework.  We investigate the behavior of our algorithm for small, medium and
large values of $L$, showing robustness of our framework in the face of
limited information.  Recall that Figure \ref{fig:MDSvsExact} shows results
for $L=70$.  Figure \ref{fig:Latency} compares the latency for $L=30$ and
$L=70$.  The clarity and accuracy of the clustering suffers with significant
reductions in information used to estimate the priors $\mu_t$.

In Figure \ref{fig:MDSvsExact}, we present snapshots of $\bm
X_t$ and $\widehat{\bm X}_t = \psi^*(\bm p_t)$ for a single simulation run.
Note that $\widehat{\bm X}_t$ is a CMDS embedding of a dissimilarity matrix
based on the posteriors $\bm p_t$.  The colors denote the final cluster
membership as determined from $k$-means clustering with $\bm X_T$.  It is
clear that the emerging cluster structure of the $\widehat{\bm X}_t$ lags
slightly behind that of $\bm X_t$ in both accuracy and clarity; comparing the
middle two figures, we can see that there are a few data points misclassified
at time $\tau_2$.  Indeed, Figure \ref{fig:ARI_MA} shows that the clustering
based on the embedded positions mirrors that possible with the true but
unobserved latent positions with a small latency.\\

\paragraph{Computing Environment}

For Experiment 1, we used R 2.14.1 (64 bit) under Ubuntu 12.0.4 on an Intel
Core i7 CPU 870 @ 2.93 GHz $\times$ 8 machine with 16 GB RAM.  For a single
run for $8$, $16$ and $32$ actors, our experiment took $190$, $788$ and $5384$
seconds respectively.  For Experiment 2, we used a Red Hat Linux cluster with
24 nodes with 24 $\times$ 2.5 MHz CPUs and 132 GB memory each.  Each Monte
Carlo replicate took a single slot.  A single replicate took approximately 3000
seconds. 

\begin{figure}[t]
    \begin{center}
        \includegraphics[width=0.7\textwidth]{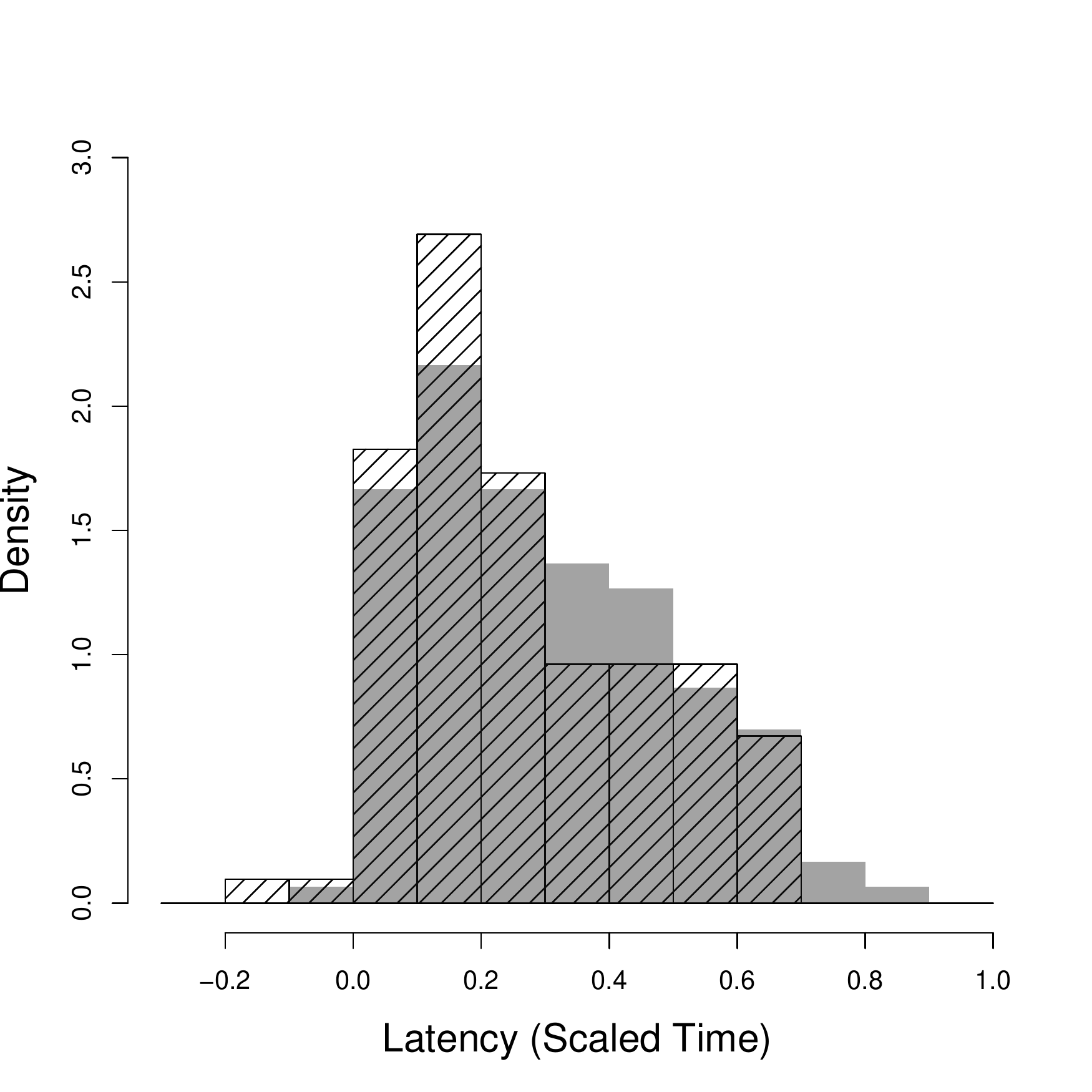}
    \end{center}
\caption{Latency ($\Delta\zeta$) distribution for $200$ Monte Carlo
    experiments.  The translucent grey histogram is based on $L=30$, and the
    cross-hatch shaded histogram is based on $L=70$.
The latency is defined as the difference between
the time $\widehat{\zeta}$ at which the moving average of the predictive ARI maintained a level of
$1-\varepsilon$ for all $t
\ge \widehat{\zeta}$ and the time $\zeta$ at
which true locations' moving average of the ARI maintained a level of $1-\varepsilon$ for all $t
\ge \zeta$. The latency can be negative, but is generally small and positive.
} \label{fig:Latency}
\end{figure}

\begin{figure}[t]
     \begin{center}
         \includegraphics[width=0.7\textwidth]{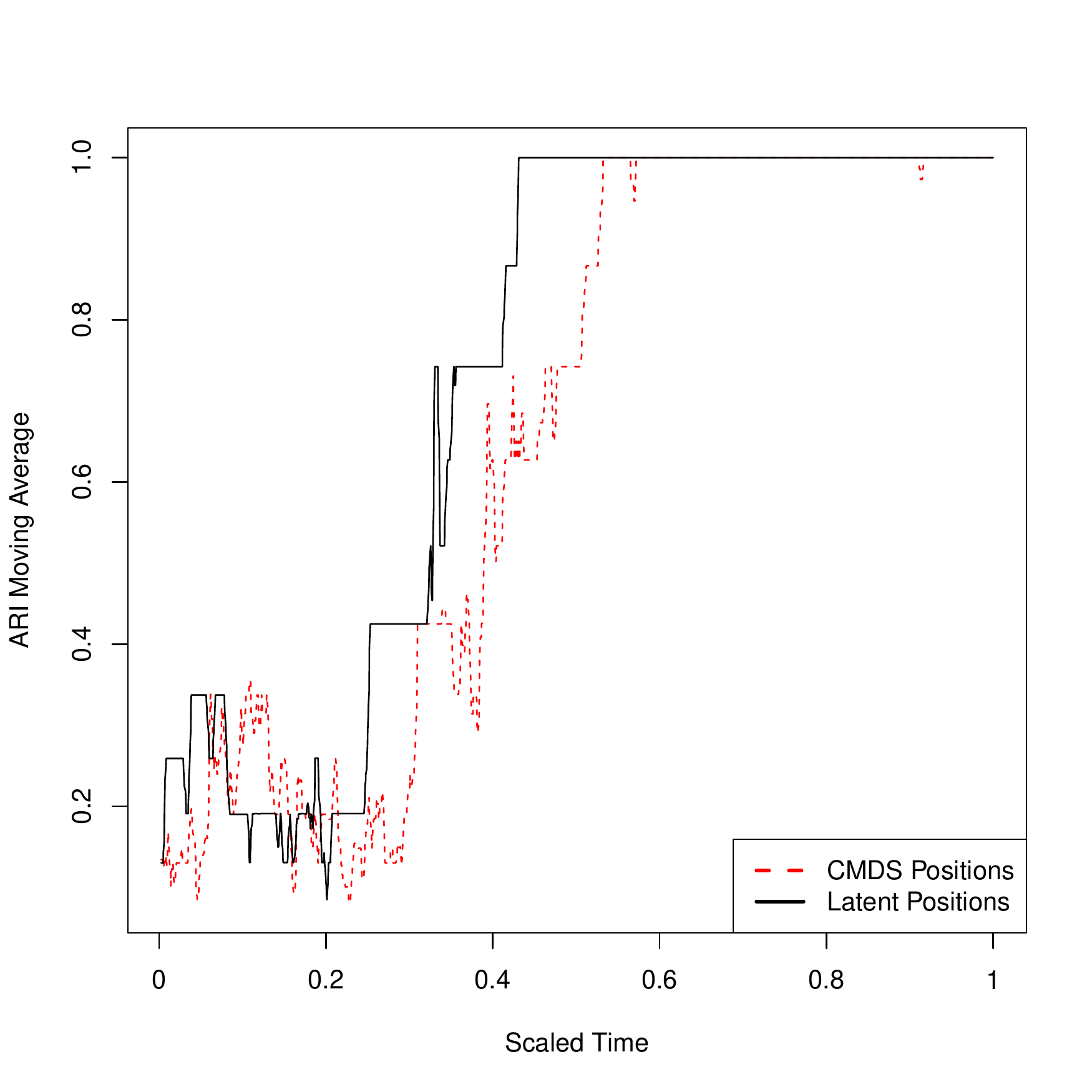}
    \end{center}
\caption{
Moving average of the $k$-means clusterings of the embedded $\widehat{\bm
X}_t$ and $\bm X(t)$ against the $k$-means clustering of $\bm X_T$.  Note that
$\tau_2$ (cf.\ Figure ~\ref{fig:MDSvsExact}) and $\eta^*$ are nearly
identical.  } \label{fig:ARI_MA}
\end{figure}

\begin{figure}[t]
    \begin{center}
        \includegraphics[width=0.7\textwidth]{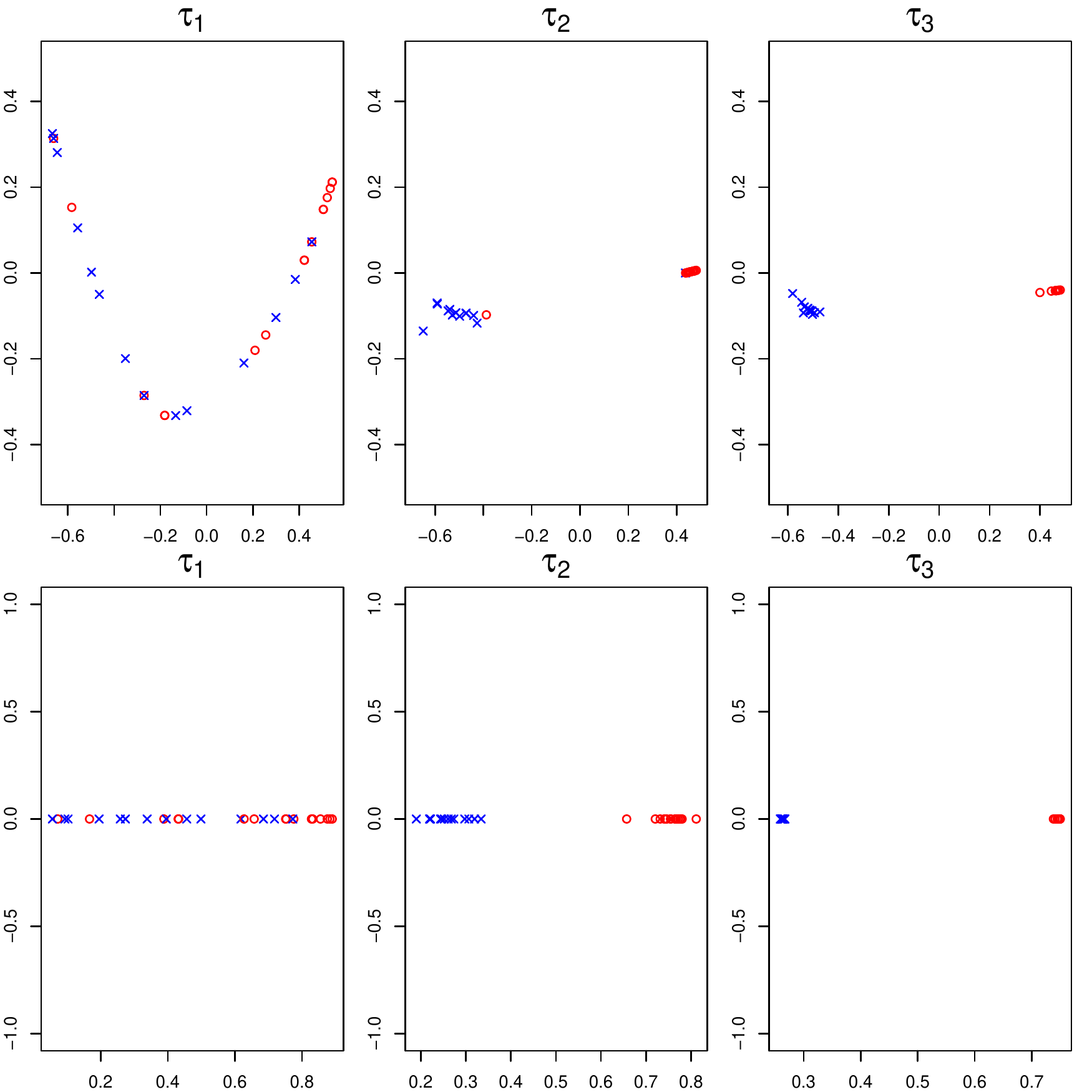}
    \end{center}
    \caption{$X_{t,i}$ versus $\widehat{X}_{t,i}$ at times $\tau_1<\tau_2<\tau_3$.
The size of the population used to estimate the prior was $L = 70$.  The first
row shows the CMDS embedded positions ($k=2$), and the second row shows the
latent positions. Due to our CMDS embedding procedure (with rotation), the 1-dimensional
embedding is the first coordinate of a 2-dimensional embedding.  We show a
2-dimensional embedding for illustration purposes.  
}
\label{fig:MDSvsExact} 
\end{figure}

\section{Conclusion and Future Work}
We have described a strategy for clustering actors based on messaging
activities. Our analysis is completed by clustering a CMDS embedding of
posteriors.  We have presented ways to simplify posterior analysis on two
levels.  The first level allows us to obtain an estimate of the
posteriors in an online manner.  The second level allows us to reduce our
analysis to studying diffusion processes, which is often a starting point for
addressing the optimal stopping problem.

We have illustrated in our numerical experiments that the assumptions used to
derive our two simplified approaches are mild enough to be useful for our
inference task at hand, i.e., clustering. 

We believe that our framework has potential for tackling the problems faced by
the social network practitioner regarding emergence of structure.  We intend to
develop a measure of confidence for our inferred latent positions.  This will
be crucial to many applications, as it will provide the decision-maker with
information about whether to act or to wait for more data to increase the
confidence in the inferred positions.  A measure of confidence would therefore
be a way to establish a stopping rule.  Noting that we took the
parameters of our model to be exogenous, we will need to explore robustness of
our inference to incorrect parameter choices and then make explicit an
algorithm for parameter estimation.  Making our algorithm more scalable is
also an area of our interest.  These areas of future work will be key to
applying our framework on substantial problems.

\section*{Acknowledgements}
This work is partially supported by a National Security Science and
Engineering Faculty Fellowship (NSSEFF), by the Acheson J. Duncan Fund for the
Advancement of Research in Statistics, and by the Johns Hopkins University
Human Language Technology Center of Excellence (JHU HLT COE).  We also thank
Dr.~Youngser Park for his technical assistance.

\begin{appendix}
\section{Motivation for the form of the differential operator $\mathcal A_t$}
\label{appendix:motivationforA}
\subsection{Bounded confidence model: an adaptation}
Our work in this paper is in part influenced by a so-called bounded confidence
model in \cite{GomezGrahamLeBoudec} which focuses on establishing a
propagation of chaos property of the interacting particles model studied
there.  When denoting the actors' latent positions $X_1(t), X_2(t),\ldots \in
[0,1]$, in the bounded confidence model, the \emph{opportunities} for (latent)
position changes that each actor experiences is modeled as a simple Poisson
process.  When there is a change at time $t$, the change is assumed to involve
precisely two actors, say, actor $i$ and actor $j$, such that their position
$X_i(t-)$ and $X_j(t-)$ differs by at most $\Delta$.  This yields an
inhomogeneity in the rate at which actors change their locations. Then, the
exact amount of change is specified by the following formula: 
\begin{align*}
    X_i(t) &=\omega X_i(t-) + (1-\omega) X_j(t-),\\
    X_j(t) &=\omega X_j(t-) + (1-\omega) X_i(t-),
\end{align*}
where $w\in (0,1)$ is a fixed constant.  Roughly speaking, upon interaction,
actor $i$ keeps $w\times 100$ percent of its original position, and is allowed
to be influenced by $(1-\omega)\times 100$ percent of the original position of
actor $j$, and vice versa.

Fix constants $\Delta \in (0,1)$ and $w \in (0,1)$.
Then, define $\mathcal{A}$ by
letting for each $\mu$ and $f$,
\begin{align*}
    \mathcal{A}(\mu)f(x) 
    = 
    2 
    \int_{|x-y|\le \Delta} 
        f(\omega x+(1-\omega)y) - f(x)) 
    \mu(dy).
\end{align*}

Studied in \cite{GomezGrahamLeBoudec} particularly is the interaction between
$\mu_t$ and $\bm X = (X_{1},\ldots, X_{n})$ where $\mu_t$ is the empirical
distribution of $\bm X$.  As shown in \cite{GomezGrahamLeBoudec}, the bounded
confidence model has an appealing feature that the parameter space for the
underlying parameters $w$ and $\Delta$ can be partitioned according to the
type of consensus that the population eventually reaches, namely, a total
consensus and a partial consensus.  In a total consensus regime, for
sufficiently large $t$, everyone is expected to gather tightly around some
fixed common point $x_0 \in [0,1]$.  On the other hand, in a partial consensus
regime, (depending on $w$ and $\Delta$), there is a finite collection of
distinct values in $[0,1]$ separated by at least $\Delta$, to exactly one of
which each actor's position is attracted.  In particular, the (asymptotic)
position of actors yields a partition of the actor set when the exact
locations of $X_1, X_2, \ldots$ are known.  Generally, $(\mu_t:t \in [0,T])$
is contracting toward for some closed convex non-empty disjoint subsets $B_1$
and $B_2$ of $[0,1]$ in the sense that for some $t \in [0,T]$, $\mu_s(B_1\cup
B_2) \ge \mu_t(B_1\cup B_2)$ for each $s \ge t$ and $\mu_T([0,1]) = 1$.

In our adaptation, for analytic tractability, we replace the indicator
function ${\bm 1}_{|x-y|\le \Delta}$ with $\psi(z) = \exp(-\frac{1}{2}z^\top
z)$, take $\mu_t$ to be an exogenous modeling element, and take $w_t$ to be
potentially time dependent, yielding the operator \begin{align}\label{dynbcm}
    \mathcal A(\mu)f(x) = 2 \int \psi(y-x) \left(f(\omega x+ (1-\omega) y) -
    f(x)\right) \mu(y) dy.  \end{align}

The second numerical experiment in Section \ref{sec:NumericalExperiments}
focuses on the case where the community
starts with no apparent clustering but as time passes, each actor becomes a
member of exactly one of clusters, where each cluster is uniquely
identified by a closed convex subset of the latent space $\mathbb R^d$.

\subsection{A quadratic Taylor series approximation}
In this work, we use a model that that captures the action in \eqref{dynbcm}
up to the second order. To begin, note that 
\begin{align*}
    f(z) = f(x) + Df(x) \cdot (z - x) +  \frac{1}{2} (z-x)^\top D^2f(x) (z-x)
    + \textsc{H.O.T.},
\end{align*}
where $Df(x)\in \mathbb R^d$ and $Df(x) \in \mathbb M_{d\times d}$ denote
respectively the gradient and the Hessian of $f$ at $x$, and \textsc{H.O.T.}
denotes the higher order terms. Suppose that $\mu_t$ is given. Now, we have 
\begin{align*}
    \mathcal A_t f(x) &:= 
    \mathcal A(\mu_t)f(x) \\
    &= 
    2 \int \psi(y-x) 
    Df(x)\cdot (1-\omega) (y-x)
    \mu_t(y) 
    dy \\
    &\qquad +  
    2 \int \psi(y-x) 
    \left(
        \frac{1}{2} (1-\omega)^2 (y-x)^\top D^2f(x) (y-x)
    \right)
    \mu_t(y) 
    dy
    +
    \textsc{H.O.T.} \\
    &= 
    \left(\sum_{k=1} b_t^k(x) \partial_k f(x) 
    + 
    \sum_{k_1} \sum_{k_2} a_t^{k_1,k_2}(x) \partial_{k_1,k_2}^2 f(x)
    \right)
    +
    \textsc{H.O.T.},
\end{align*}
where $b_t(x) \in \mathbb R^d$ and $a_t(x) \in \mathbb R^{d\times
d}$ are given by the following:
\begin{align*}
    &b_t^k(x) = 2(1-\omega_t) \int \psi(y-x) (y-x)_k \mu_t(y) dy, \\
    &a_t^{k_1,k_2}(x) = (1-\omega_t)^2 \int \psi(y-x) (y-x)_{k_1}
    (y-x)_{k_2} \mu_t(y) dy.
\end{align*}
Dropping the term associated with \textsc{H.O.T.}, we obtain the following:
\begin{align*}
    \mathcal A_t f(x) 
    = 
    \left(\sum_{k} b_t^k(x) \partial_k f(x) 
    + 
    \sum_{k_1} \sum_{k_2} a_t^{k_1,k_2}(x) \partial_{k_1,k_2}^2 f(x)
    \right).
\end{align*}

\section{The mixture projection filter formula}
\label{appendix:mixtureprojectionformula}
\subsection{Proof of Theorem \ref{thm:ProjFilterZakai}}
For each $\phi_{r}$, we see that 
\begin{align*}
    \langle \phi_{r}, d{p}_{t,i} \rangle 
    = 
    \sum_{c} \langle \phi_{r}, \phi_{c} \rangle d{W}_{t,i,c}
    = e_{r}^\top P d{W}_{t,i}.
\end{align*}
We first consider the second term of the right side of
\eqref{eqn:ProjFilterZakai}. \begin{align*}
    dH_{t,i}(x) 
    &:=
  \sum_{j\neq i} 
    \left(
    \frac{\phi_r(x) 
    (
    p_{t,i}(x) \lambda_{t,i\rightarrow j}(x)
    +
    p_{t,j}(x) \lambda_{t,j\rightarrow i}(x))
    }{\lambda_{t,ij}} - p_{t,i}(x)
    \right)
    dM_{t,ij} \\
    &:=
  \sum_{j\neq i} 
    \left(
    \frac{\phi_r(x) 
    (
    p_{t,i}(x) \lambda_{t,i\rightarrow j}(x)
    +
    p_{t,j}(x) \lambda_{t,j\rightarrow i}(x))
    }{\lambda_{t,ij}} - p_{t,i}(x)
    \right)
    \left( dN_{t,ij} - \lambda_{t,ij} dt \right).
\end{align*}
Now, we have that
\begin{align*}
    &\quad\int \phi_r(x) 
    (p_{t,i}(x) \lambda_{t,i\rightarrow j}(x) + p_{t,j}(x) \lambda_{t,j\rightarrow i}(x))
    dx\\
    &= (\lambda_{t,i} \lambda_{t,j}/2)
    \int \phi_r(x) 
    (p_{t,i}(x) p_{t,j}(x) + p_{t,j}(x) p_{t,i}(x)) dx \\
    &= (\lambda_{t,i}\lambda_{t,j}/2) 2 \langle \phi_r,p_{t,i}p_{t,j}\rangle \\
    &= \lambda_{t,i}\lambda_{t,j} \langle \phi_r,p_{t,i}p_{t,j}\rangle,
\end{align*}
and that
\begin{align*}
    \lambda_{t,ij} 
    = 
    \lambda_{t,i}\lambda_{t,j} \langle p_{t,i}, p_{t,j} \rangle 
    = 
    \lambda_{t,i}\lambda_{t,j} 
    \sum_{k_1=1}^K \sum_{k_2=1}^K 
        W_{t,i}^{k_1} W_{t,j}^{k_2} \langle \phi_{k_1}, \phi_{k_2} \rangle
    =
    \lambda_{t,i}\lambda_{t,j} W_{t,i}^\top P W_{t,j}.
\end{align*}
Hence,
\begin{align*}
    \langle \phi_r, dH_{t,i} \rangle 
    &= 
    \sum_{j\neq i } 
    \left(
    \dfrac{\langle \phi_r, p_{t,i}p_{t,j} \rangle }{W_{t,i}^\top P W_{t,j}}
    - 
    \langle \phi_r, p_{t,i} \rangle
    \right)
    \left(
    dN_{t,ij} 
    - 
    \lambda_{t,i}\lambda_{t,j} W_{t,i}^\top P W_{t,j} dt 
    \right)\\
    &= 
    \sum_{j\neq i } 
    \left(
    \dfrac{W_{t,i}^\top S_r W_{t,j}}{W_{t,i}^\top P W_{t,j}}
    - 
    e_r^\top P W_{t,i}
    \right)
    \left(
    dN_{t,ij} 
    - 
    \lambda_{t,i}\lambda_{t,j} W_{t,i}^\top P W_{t,j} dt 
    \right).
\end{align*}

Next, for the first term of the right side of
\eqref{eqn:ProjFilterZakai}, we have 
\begin{align*}
    \langle \phi_{r}, \mathcal A_{t,i}^* p_{t,i} \rangle 
    &= \sum_{\ell} q_{t,\ell} \langle \mathcal A_{t,i,\ell} \phi_{r}, p_{t,i} \rangle \\
    &= \sum_{\ell} q_{t,\ell} \sum_{c}  \langle \mathcal A_{t,i,\ell} \phi_{r}, \phi_c \rangle W_{t,i,c}\\
    &= \sum_{\ell} q_{t,\ell} \sum_{c}  \langle \mathcal A_{t,i,\ell} \phi_{r}, \phi_c \rangle W_{t,i,c}\\
    &= \sum_{\ell} q_{t,\ell} e_r^\top R_{t,i,\ell} W_{t,i}.
\end{align*}
In summary, for each $r$, we have 
\begin{align*}
    e_r^\top P d{W}_{t,i}
    &= 
    e_r^\top R_t W_{t,i} dt 
    +
    \sum_{j\neq i } 
    \left(
    \dfrac{W_{t,i}^\top S_r W_{t,j}}{W_{t,i}^\top P W_{t,j}}
    - 
    e_r^\top P W_{t,i}
    \right)
    \left(
    dN_{t,ij} 
    - 
    \lambda_{t,i}\lambda_{t,j} W_{t,i}^\top P W_{t,j} dt 
    \right),
\end{align*}
and our claim follows from this.

\subsection{Preliminary lemmas}
This section contains two formulas to be used in the next section.
Our result and proof in Lemma \ref{lem:product-density-formula}
is stated in the same notation as in 
Lemma \ref{lem:algebra-formula}. Recall that
$\phi(z;\vartheta, \gamma ) \propto \phi(\gamma^{-1}(z-\vartheta))$.

\begin{lemma}\label{lem:algebra-formula}
    Let $\{\vartheta_\ell\} \subset \mathbb R^d$ and $\{\gamma_\ell\} \subset
    \mathbb R_+$.  Then, 
\begin{align*}
&\quad \sum_{\ell} \left\|\gamma_\ell^{-1} (x-\vartheta_\ell)\right\|^2\\
&=
\left(\sum_{\ell} \gamma_\ell^{-2}\right)
\left\| x - 
\sum_\ell \left(\frac{\gamma_\ell^{-2}}{\sum_{m} \gamma_m^{-2}} \right)\vartheta_\ell 
\right\|^2
-\frac{\bm 1^\top \left(\Gamma * (\Theta - \diag(\Theta) \bm 1^{\top})\right)\bm 1 }{\sum_n \gamma_n^{-2}}
,
\end{align*}
where $\Theta$ is the Gram matrix for $(\theta_\ell)$ and $\Gamma$ is the
matrix whose $(r,c)$-entry is $\gamma_r^{-2}\gamma_c^{-2}$.  
\end{lemma}
\begin{proof}
Let $C = \sum_{\ell} \gamma_\ell^{-2}$ and for each $\ell$, let 
$\rho_\ell = \gamma_\ell^{-2}/C$. First, note that 
\begin{align*}
&\quad\sum_{\ell} \left\|\gamma_\ell^{-1} (x-\vartheta_\ell)\right\|^2
= 
\sum_{\ell} 
\gamma_\ell^{-2}
\left(x^\top x - 2 x^\top \vartheta_\ell + \vartheta_\ell^\top \vartheta_\ell\right)\\
&= 
\left(\sum_{\ell} \gamma_\ell^{-2}\right)
\|x\|^2 
- 
2 x^\top \left( \sum_\ell \gamma_\ell^{-2}\vartheta_\ell\right) 
+ \sum_\ell \vartheta_\ell^\top \gamma_\ell^{-2}\vartheta_\ell\\
&= 
\left(\sum_{\ell} \gamma_\ell^{-2}\right)
\left(
\|x\|^2 
- 
2 x^\top \left( \sum_\ell \rho_\ell\vartheta_\ell\right) 
+ 
\|\sum_\ell \rho_\ell \vartheta_\ell \|^2
-
\|\sum_\ell \rho_\ell \vartheta_\ell \|^2
\right)
+ \sum_\ell \gamma_\ell^{-2}\vartheta_\ell^\top \vartheta_\ell\\
&= 
C\left\| x- \sum \rho_\ell \vartheta_\ell \right\|^2 
- 
C \| \sum_\ell \rho_\ell \vartheta_\ell \|^2
+ \sum_\ell \gamma_\ell^{-2}\vartheta_\ell^\top \vartheta_\ell.
\end{align*}
Now, 
\begin{align*}
&\quad C \| \sum_\ell \rho_\ell \vartheta_\ell \|^2
- 
\sum_\ell \gamma_\ell^{-2}\vartheta_\ell^\top \vartheta_\ell \\
&=
C 
\sum_r \sum_c
\frac{\gamma_r^{-2}}{C} 
\frac{\gamma_c^{-2}}{C} 
\vartheta_r^\top \vartheta_c
-
\sum_r
\gamma_r^{-2}
\vartheta_r^\top \vartheta_r\\
&=
\frac{1}{C}
\sum_r \sum_{c\neq r}
\gamma_r^{-2} 
\gamma_c^{-2} 
\vartheta_r^\top \vartheta_c
+
\sum_r
\gamma_r^{-2}
(
\gamma_r^{-2}/C
-
1
)
\vartheta_r^\top \vartheta_r\\
&=
\frac{1}{C}
\left(
\sum_r \sum_{c\neq r}
\gamma_r^{-2} 
\gamma_c^{-2} 
\vartheta_r^\top \vartheta_c
+
\sum_r
\gamma_r^{-2}
(
\gamma_r^{-2}
-
C
)
\vartheta_r^\top \vartheta_r
\right)\\
&=
\frac{1}{C}
\left(
\sum_r \sum_{c\neq r}
\gamma_r^{-2} 
\gamma_c^{-2} 
\vartheta_r^\top \vartheta_c
-
\sum_r
\gamma_r^{-2}
\sum_{c\neq r} 
\gamma_c^{-2}
\vartheta_r^\top \vartheta_r
\right)\\
&=
\frac{1}{\sum_{\ell} \gamma_\ell^{-2}}
\left(
\sum_r \sum_{c}
\gamma_r^{-2} 
\gamma_c^{-2} 
\left(
\vartheta_r^\top \vartheta_c
-
\vartheta_r^\top \vartheta_r
\right)
\right).
\end{align*}
Our claim follows from this. 
\end{proof}
\begin{lemma}\label{lem:product-density-formula}
Let $\phi$ be the standard multivariate normal density defined 
on $\mathbb R^d$. 
Also, fix a sequence $\{\gamma_m\}_{m=1}^M \subset \mathbb R_+$,
and a sequence $\{\vartheta_\ell\}_{m=1}^M \subset \mathbb R^d$. 
    \begin{align*}
        \prod_{m} \phi(x;\vartheta_m,\gamma_m)
        &= 
        \left(
        \frac{2\pi/\prod_{m}(2\pi)}{\sum_{m}\gamma_m^{-2} 
        \prod_{m} \gamma_m^{2}}
        \right)^{d/2} \\
        &\qquad\cdot
        \exp\left(
        \frac{1}{2} 
        \frac{\bm 1^\top \left(\Gamma * (\Theta - \diag(\Theta) \bm 1^{\top})\right)\bm 1 }{\sum_n \gamma_n^{-2}}
        \right)\\
        &\qquad \cdot 
        \phi\left(x ;
        \sum_\ell \left(\frac{\gamma_\ell^{-2}}{\sum_{m} \gamma_m^{-2}} \right)\vartheta_\ell 
        ,\left(\sum_m \gamma_m^{-2}\right)^{-1/2}\right).
    \end{align*}
\end{lemma}
\begin{proof}
 Using Lemma \ref{lem:algebra-formula}, we see that
\begin{align*}
\prod_{m} \phi(x;\vartheta_m,\gamma_m)
&= 
\prod_{m} \frac{1}{\left(\sqrt{2\pi\gamma_m^2}\right)^d} 
\exp
\left(
-\frac{1}{2}
\left\|
\frac{x - \vartheta_m}{\gamma_m}
\right\|^2
\right)
\\
&=
\frac{1}{\left(\sqrt{\prod_m2\pi\gamma_m^2}\right)^d} 
\exp
\left(
-\frac{1}{2}
\sum_m
\left\|
\frac{x - \vartheta_m}{\gamma_m}
\right\|^2
\right)\\
&=
\frac{1}{\left(\sqrt{\prod_m 2\pi\gamma_m^2}\right)^d} 
\exp
\left(
-\frac{1}{2}
\left\| x - 
\sum_\ell \left(\frac{\gamma_\ell^{-2}}{\sum_{m} \gamma_m^{-2}} \right)\vartheta_\ell 
\right\|^2 \Big /
\left(\sum_{\ell} \gamma_\ell^{-2}\right)
\right)\\
&\qquad 
\cdot \exp
\left(
\frac{1}{2}
\frac{\bm 1^\top \left(\Gamma * (\Theta - \diag(\Theta) \bm 1^{\top})\right)\bm 1 }{\sum_n \gamma_n^{-2}}
\right)\\
&=
\frac{
\left(\sqrt{2\pi\left(\sum_{\ell} \gamma_\ell^{-2}\right)^{-1}
}\right)^d}{\left(\sqrt{\prod_m2\pi\gamma_m^2}\right)^d} 
\phi\left(x;
\sum_\ell \left(\frac{\gamma_\ell^{-2}}{\sum_{m} \gamma_m^{-2}}
\right)\vartheta_\ell,
\left(\sum_{\ell} \gamma_\ell^{-2}\right)^{-1/2}
\right)\\
&\qquad 
\cdot \exp
\left(
\frac{1}{2}
\frac{\bm 1^\top \left(\Gamma * (\Theta - \diag(\Theta) \bm 1^{\top})\right)\bm 1 }{\sum_n \gamma_n^{-2}}
\right)\\
&=
\left(
\frac{2\pi/\prod_m2\pi}{\sum_{\ell} \gamma_\ell^{-2} \prod_{\ell}\gamma_\ell^2} 
\right)^{d/2}
\phi\left(x;
\sum_\ell \left(\frac{\gamma_\ell^{-2}}{\sum_{m} \gamma_m^{-2}}
\right)\vartheta_\ell,
\left(\sum_{\ell} \gamma_\ell^{-2}\right)^{-1/2}
\right)\\
&\qquad 
\cdot \exp
\left(
\frac{1}{2}
\frac{\bm 1^\top \left(\Gamma * (\Theta - \diag(\Theta) \bm
1^{\top})\right)\bm 1 }{\sum_m \gamma_m^{-2}}
\right).
\end{align*}
\end{proof}

\subsection{Formula for $R_{t,i,\ell}$ in a multivariate normal density case}
Here, we assume, as done in Theorem \ref{thm:ProjFilterZakai}, that $\mu_t(y)
= \sum_{\ell} q_{t,\ell} \phi_\ell(y)$ and $p_{t,i}(x) = \sum_{\ell}
W_{t,i,\ell} \phi_\ell(y)$, where for simplicity, we have written
$\phi_\ell(z) := \phi(z;\theta_\ell, s) \propto \phi(s^{-1}(z-\theta_\ell))$.  In
this section, we fix $\phi$ to be the standard multivariate normal density
defined on $\mathbb R^d$ and recall that $\psi(z) = \phi(z)/\phi(0)$.  Also,
we fix $s \in \mathbb R_+$, and a sequence $\{\theta_\ell\}\subset \mathbb R^d$. 

%%%% LEMMA B.3
\begin{lemma}\label{lem:micro-b-a}
Fix $\theta_\ell$, $x \in \mathbb R^d$ and $s \in \mathbb R_+$. 
For each $k$,
\begin{align}
    &b_{t,i,\ell}^k(x) 
    =
    -2 (1-\omega_{t,i}) 
    (x - \theta_\ell)_k
    \frac{\sigma_{t,i}^2}{\sigma_{t,i}^2+\alpha_{t,\ell}^2}
    \sqrt{(2\pi\sigma_{t,i}^2)^{d}}
    \phi(x;\theta_\ell,
    (\sigma_{t,i}^2+\alpha_{t,\ell}^2)^{1/2})
    ,\label{eqn:micro-b}  \\
    &a_{t,i,\ell}^{k_1,k_2}(x) 
    =
    (1-\omega_{t,i})^2 \sqrt{(2\pi \sigma_{t,i}^2)^d} 
    \phi(x;\theta_\ell,(\sigma_{t,i}^2 + \alpha_{t,\ell}^2)^{1/2}) 
    \nonumber \\
    &\qquad\qquad\qquad \cdot
    \left( 
    \left(\frac{\sigma_{t,i}^2}{\sigma_{t,i}^2+\alpha_{t,\ell}^2}\right)^2
    (x - \theta_\ell)_{k_1} (x - \theta_\ell)_{k_2}
    +
    \bm 1\{k_1=k_2\} 
    \frac{\sigma_{t,i}^2\alpha_{t,\ell}^2}{\sigma_{t,i}^2+\alpha_{t,\ell}^2} \right). \label{eqn:micro-a}
\end{align}
\end{lemma}
\begin{proof}
Let
\begin{align*}
    &v_{t,i,\ell} = 1/(\sigma_{t,i}^{-2}+\alpha_{t,\ell}^{-2}) = \frac{\sigma_{t,i}^2 \alpha_{t,\ell}^2}{\sigma_{t,i}^2 + \alpha_{t,\ell}^2},\\
    &c_{t,i,\ell} = \frac{\alpha_{t,\ell}^2}{\sigma_{t,i}^2 +
    \alpha_{t,\ell}^2} x
    +
    \frac{\sigma_{t,i}^2}{\sigma_{t,i}^2 + \alpha_{t,\ell}^2} \theta_\ell,
\end{align*}
where to simplify the expression of $c_{t,i,\ell}$,  we have used the fact that 
\begin{align*}
    \sigma_{t,i}^2 \alpha_{t,\ell}^2 \times (\sigma_{t,i}^{-2} + \alpha_{t,\ell}^{-2}) = (\sigma_{t,i}^{2} + \alpha_{t,\ell}^{2}).
\end{align*}
Also, note that
\begin{align*}
   &\quad \left( \sigma_{t,i}^{-2} + \alpha_{t,\ell}^{-2} \right)^{-1} 
    \bm 1^\top \left(\Gamma * (\Theta - \diag(\Theta) \bm 1^\top)\right) \bm 1\\
   &= \left( \sigma_{t,i}^{-2} + \alpha_{t,\ell}^{-2} \right)^{-1} 
   \bm 1^\top 
   \left(
   \begin{bmatrix}
       \sigma_{t,i}^{-4} & \sigma_{t,i}^{-2} \alpha_{t,\ell}^{-2} \\
       \sigma_{t,i}^{-2} \alpha_{t,\ell}^{-2} & \alpha_{t,\ell}^{-4} 
   \end{bmatrix}
   *
   \begin{bmatrix}
       0 & x^\top \theta_\ell -x^\top x\\
       x^\top \theta_\ell -\theta_\ell^\top \theta_\ell & 0
   \end{bmatrix}
   \right)
   \bm 1\\
   &= \left( \sigma_{t,i}^{-2} + \alpha_{t,\ell}^{-2} \right)^{-1} 
   \bm 1^\top 
   \left(
   \begin{bmatrix}
       \sigma_{t,i}^{-2}\alpha_{t,\ell}^{-2} & \sigma_{t,i}^{-2} \alpha_{t,\ell}^{-2} \\
       \sigma_{t,i}^{-2} \alpha_{t,\ell}^{-2} & \sigma_{t,i}^{-2} \alpha_{t,\ell}^{-2} 
   \end{bmatrix}
   *
   \begin{bmatrix}
       -x^\top x & x^\top \theta_\ell \\
       x^\top \theta_\ell & -\theta_\ell^\top \theta_\ell
   \end{bmatrix}
   \right)
   \bm 1\\
   &=  \left( \sigma_{t,i}^{-2} + \alpha_{t,\ell}^{-2} \right)^{-1} 
   \left( \sigma_{t,i}^{-2} \alpha_{t,\ell}^{-2}\right) 
   (-\left\| x- \theta_\ell\right\|^2)\\
   &=
   - \left( \sigma_{t,i}^{2} + \alpha_{t,\ell}^{2} \right)^{-1} 
   \left\| x-\theta_\ell \right\|^2.
\end{align*}
Using Lemma \ref{lem:product-density-formula} with $x=y$, $\vartheta_1=x$ and
$\vartheta_2=\theta_\ell$, we see that
\begin{align*}
    &\exp\left(
    -\frac{1}{2}
    \frac{1}{\sigma_{t,i}^2} 
    \|y-x\|^2
    \right)
    \exp\left(
    -\frac{1}{2}
    \frac{1}{\alpha_{t,\ell}^2} 
    \|y-\theta_\ell\|^2
    \right)\Big/\sqrt{(2\pi \alpha_{t,\ell}^2)^d} 
    \\
    = &
    \left(
    2\pi \sigma_{t,i}^2
    \right)^{d/2}
    \left(
    \frac{2\pi/(2\pi)^2}{ \sigma_{t,i}^2 \alpha_{t,\ell}^2
    (\sigma_{t,i}^{-2}+\alpha_{t,\ell}^{-2})}
    \right)^{d/2}
    \exp
    \left(
    -\frac{1}{2} 
    \left\| 
    \frac{x-\theta_\ell}{\sqrt{\sigma_{t,i}^2 + \alpha_{t,\ell}^2}} 
    \right\|^2
    \right)
    \frac{
    \exp
    \left(
    -\frac{ \left\| y- c_{t,i,\ell} \right\|^2 }{2v_{t,i,\ell}} 
    \right)}{(\sqrt{2\pi v_{t,i,\ell}})^d}\\
    = &
    \left(2\pi
    \sigma_{t,i}^2
    \right)^{d/2}
    \frac{
    \exp
    \left(
    -\frac{\left\|x-\theta_\ell\right\|^2}{2(\sigma_{t,i}^2 + \alpha_{t,\ell}^2)} 
    \right)
    }{\left(\sqrt{2\pi (\sigma_{t,i}^{2}+\alpha_{t,\ell}^{2})}\right)^{d}}
    \frac{
    \exp
    \left(
    -\frac{ \left\| y- c_{t,i,\ell} \right\|^2 }{2v_{t,i,\ell}} 
    \right)}{(\sqrt{2\pi v_{t,i,\ell}})^d}.
\end{align*}
Then, for our claim in \eqref{eqn:micro-b}, it is enough to see that
\begin{align*}
&\quad \int 
    \frac{1}{(\sqrt{2\pi v_{t,i,\ell}})^d}
    \exp
    \left(
    -\frac{ \left\| y- c_{t,i,\ell} \right\|^2 }{2v_{t,i,\ell}}
    \right)
    (y - x )
    dy\\
&= c - x \\
&= \left(
    \frac{\alpha_{t,\ell}^2}{\sigma_{t,i}^2 + \alpha_{t,\ell}^2} x 
    +
   \frac{\sigma_{t,i}^2}{\sigma_{t,i}^2 + \alpha_{t,\ell}^2} \theta_\ell 
   \right)
   - x \\
&= \frac{\sigma_{t,i}^2}{\sigma_{t,i}^2 + \alpha_{t,\ell}^2} 
   \left( \theta_\ell - x\right).
\end{align*}
Next,
we show our claim in \eqref{eqn:micro-a}. Hereafter, to ease our notation, we write $c$
for $c_{t,i,\ell}$. 
First, for $k_1\neq k_2$, we have
\begin{align*}
&\quad \int 
    \frac{1}{(\sqrt{2\pi v_{t,i,\ell}})^d}
    \exp
    \left(
    -\frac{ \left\| y- c \right\|^2 }{2v_{t,i,\ell}} 
    \right)
    y_{k_1} y_{k_2}
    dy\\
 &= 
 \int 
    \frac{1}{\sqrt{2\pi v_{t,i,\ell}}}
    \exp
    \left(
    -\frac{ (y_{k_1}- c_{k_1})^2 }{2v_{t,i,\ell}} 
    \right) y_{k_1} dy_{k_1}
    \cdot
 \int 
    \frac{1}{\sqrt{2\pi v_{t,i,\ell}}}
    \exp
    \left(
    -\frac{ (y_{k_2}- c_{k_2})^2 }{2v_{t,i,\ell}} 
    \right) y_{k_2} dy_{k_2}\\
&= 
c_{k_1} c_{k_2},
\end{align*}
and hence,
\begin{align*}
&\quad \int 
    \frac{1}{(\sqrt{2\pi v_{t,i,\ell}})^d}
    \exp
    \left(
    -\frac{ \left\| y- c \right\|^2 }{2v_{t,i,\ell}} 
    \right)
    (y-x)_{k_1} (y-x)_{k_2}
    dy \\
&= c_{k_1} c_{k_2} - x_{k_1} c_{k_2} - c_{k_1} x_{k_2} + x_{k_1} x_{k_2} \\
&= (c - x)_{k_1}(c - x)_{k_2} \\ 
&=\frac{\sigma_{t,i}^2}{\sigma_{t,i}^2 + \alpha_{t,\ell}^2} (\theta_\ell - x)_{k_1}
\frac{\sigma_{t,i}^2}{\sigma_{t,i}^2 + \alpha_{t,\ell}^2} (\theta_\ell - x)_{k_2}\\
&=\left(\frac{\sigma_{t,i}^2}{\sigma_{t,i}^2 + \alpha_{t,\ell}^2}\right)^2
(x - \theta_\ell)_{k_1} (x - \theta_\ell)_{k_2}.
\end{align*}
On the other hand, for $k_1=k_2$, we have
\begin{align*}
&\quad \int 
    \frac{1}{(\sqrt{2\pi v_{t,i,\ell}})^d}
    \exp
    \left(
    -\frac{ \left\| y- c \right\|^2 }{2v_{t,i,\ell}} 
    \right)
    y_{k_1} y_{k_2}
    dy\\
&= 
 \int 
    \frac{1}{(\sqrt{2\pi v_{t,i,\ell}})^d}
    \exp
    \left(
    -\frac{ (y- c_{k_1})^2 }{2v_{t,i,\ell}} 
    \right) y^2 dy\\
&= v_{t,i,\ell} + c_{k_1}^2 \\
&= \frac{\sigma_{t,i}^2 \alpha_{t,\ell}^2}{\sigma_{t,i}^2 + \alpha_{t,\ell}^2}
+ 
\left(\frac{\alpha_{t,\ell}^2}{\sigma_{t,i}^2 + \alpha_{t,\ell}^2} x_k 
+
\frac{\sigma_{t,i}^2}{\sigma_{t,i}^2 + \alpha_{t,\ell}^2} \theta_{\ell,k}
\right)^2
\end{align*}
and so, we have
\begin{align*}
&\quad \int 
    \frac{1}{(\sqrt{2\pi v_{t,i,\ell}})^d}
    \exp
    \left(
    -\frac{ \left\| y- c \right\|^2 }{2v_{t,i,\ell}} 
    \right)
    (y-x)_{k_1} (y-x)_{k_2}
    dy \\
    &= v_{t,i,\ell} + c_{k_1} c_{k_2} - x_{k_1} c_{k_2} - c_{k_1} x_{k_2} + x_{k_1} x_{k_2} \\
    &= v_{t,i,\ell} + (c - x)_{k_1}(c - x)_{k_2} \\ 
&= 
\frac{\sigma_{t,i}^2 \alpha_{t,\ell}^2}{\sigma_{t,i}^2 + \alpha_{t,\ell}^2} +
\frac{\sigma_{t,i}^2}{\sigma_{t,i}^2 + \alpha_{t,\ell}^2} (\theta_\ell - x)_{k_1}
\frac{\sigma_{t,i}^2}{\sigma_{t,i}^2 + \alpha_{t,\ell}^2} (\theta_\ell - x)_{k_2}\\
&=
\frac{\sigma_{t,i}^2 \alpha_{t,\ell}^2}{\sigma_{t,i}^2 + \alpha_{t,\ell}^2} +
\left(\frac{\sigma_{t,i}^2}{\sigma_{t,i}^2 + s^2}\right)^2
(x - \theta_\ell)_{k_1} (x - \theta_\ell)_{k_2}.
\end{align*}
Our claim in \eqref{eqn:micro-a} follows. 
\end{proof}

For Lemma \ref{lem:b} and Lemma \ref{lem:a}, 
by $\Theta_{\ell,r,c}$, 
we denote the Gram matrix for $(\theta_\ell,\theta_r,\theta_c)$, and 
define $\Gamma_{t,i}$ to be as in Lemma \ref{lem:algebra-formula} for
$\gamma_1^2 = \sigma_{t,i}^2 + \alpha_{t,\ell}^2$, $\gamma_2^2 = s^2$,
$\gamma_3^2= s^2$.
Let
\begin{align*}
    &C_0 = (\theta_\ell(\sigma_{t,i}^2+\alpha_{t,\ell}^2)^{-1}+\theta_rs^{-2}+
    \theta_cs^{-2})\Large/
    \left(
    (\sigma_{t,i}^2+\alpha_{t,\ell}^2)^{-1}+s^{-2}+s^{-2}
    \right),\\
    &C_1 = \frac{1}{\sigma_{t,i}^2+\alpha_{t,\ell}^2}+\frac{1}{s^2}+\frac{1}{s^2} 
    = \frac{s^2+2\sigma_{t,i}^2+2\alpha_{t,\ell}^2}{(\sigma_{t,i}^2+ \alpha_{t,\ell}^2) s^2},\\
    &C_2 = \left(\frac{1/(2\pi)^2}{s^4+2
    s^{2}(\sigma_{t,i}^2+\alpha_{t,\ell}^2)}\right)^{d/2}
\exp\left( \frac{1}{2} \frac{1}{C_1}
{\bm 1^\top \left(\Gamma_{t,i} * (\Theta_{\ell,r,c} - \diag(\Theta_{\ell,r,c}) \bm 1^{\top})\right)\bm 1 }
\right).
\end{align*}
To simplify our notation, we let 
\begin{align*}
\xi_k &= 
\bm 1^\top
\left(
\begin{bmatrix}
    s^2 \\
    -(\sigma_{t,i}^2+\alpha_{t,\ell}^2+s^2) \\
    \sigma_{t,i}^2+\alpha_{t,\ell}^2
\end{bmatrix}
\begin{bmatrix}
    -2 & 1 & 1
\end{bmatrix}
* 
\begin{bmatrix}
    \theta_{\ell,k} \\
    \theta_{r,k} \\
    \theta_{c,k} 
\end{bmatrix}
\begin{bmatrix}
    \theta_{\ell,k}, \theta_{r,k}, \theta_{c,k} 
\end{bmatrix}
\right) \bm 1,\\
\Xi &= {\bm 1^\top \left(\Gamma_{t,i} * (\Theta_{\ell,r,c} -
\diag(\Theta_{\ell,r,c}) \bm 1^{\top})\right)\bm 1 }.
\end{align*}
Define and note 
\begin{align*}
    \xi := \sum_{k=1}^K \xi_k 
    = 
\bm 1^\top
\left(
\begin{bmatrix}
    s^2 \\
    -(\sigma_{t,i}^2+\alpha_{t,\ell}^2+s^2) \\
    \sigma_{t,i}^2+\alpha_{t,\ell}^2
\end{bmatrix}
\begin{bmatrix}
    -2 & 1 & 1
\end{bmatrix}
* 
\Theta_{\ell,r,c}
\right) \bm 1.
\end{align*}
Also, denote by $h_{\ell,r,c}(x)$ the multivariate 
normal density defined on $\mathbb R^d$ such that its
mean vector is $C_0$ and its covariance matrix is $C_1^{-1}\bm I$. 
For $f \in C(\mathbb R)$ and $k=1,\ldots, d$, we write
\begin{align*}
    \langle f(x_k) \rangle_{\ell,r,c}
    = 
    \int_{\mathbb R^d} f(x_k) h_{\ell,r,c}(x)dx,
\end{align*}
and note that in particular, 
\begin{align*}
    &\langle x_k \rangle_{\ell,r,c} = C_{0,k},\\
    &\langle x_k^2 \rangle_{\ell,r,c} = C_{0,k}^2 + C_1^{-1},\\
    &\langle x_k^3 \rangle_{\ell,r,c} = C_{0,k}^3 + 3 C_{0,k} C_1^{-1},\\
    &\langle x_k^3 \rangle_{\ell,r,c} = C_{0,k}^4 + 6 C_{0,k}^2 C_1^{-1} + 3 C_1^{-2}. 
\end{align*}

Starting from \eqref{eqn:AtiellVSRtiellrc}, it is easy to see that 
\begin{align}
    \langle \mathcal A_{t,i,\ell}\phi_r,\phi_c \rangle 
    = 
    \sum_{k=1}^d \langle b_{t,i,\ell}^k \partial_k \phi_r, \phi_c \rangle
    + 
    \sum_{k_1=1}^{d} \sum_{k_2=1}^d  
    \langle a_{t,i,\ell}^{k_1,k_2} \partial_{k_1} \partial_{k_2} \phi_r,
    \phi_c \rangle,
    \label{eqn:Atiellrc}
\end{align}
and as a matter of definition, we have 
\begin{align*}
    R_{t,i,\ell} = (R_{t,i,\ell,rc})_{r,c=1}^K = ( \langle \mathcal
    A_{t,i,\ell}\phi_r,\phi_c \rangle )_{r,c=1}^K \in \mathbb M_{K,K}.
\end{align*}

Lemma \ref{lem:b} and Lemma \ref{lem:a} are associated, respectively, with the
first and the second terms appearing in the right side of \eqref{eqn:Atiellrc}. 
%%%%% LEMMA B.4
\begin{lemma}\label{lem:b}
For each $\ell, r, c$ and $i,t$, we have
\begin{align*}
    &\quad \sum_{k=1}^d\langle b_{t,i,\ell}^k\partial_k\phi_r,\phi_c\rangle\\
    &=(2\pi \sigma_{t,i}^2)^{d/2}
    \frac{(1-\omega_{t,i}) \sigma_{t,i}^2}{ (\sigma_{t,i}^2+\alpha_{t,\ell}^2)}
    \left(
\frac{2\sigma_{t,i}^2+2\alpha_{t,\ell}^2 }{s^2+2\sigma_{t,i}^2+2\alpha_{t,\ell}^2}
+
\frac{1}{s^2}
\frac{ 2\sigma_{t,i}^2+2\alpha_{t,\ell}^2 }{(s^2+2\sigma_{t,i}^2+2\alpha_{t,\ell}^2)^2}
\xi \right)\\
&\qquad 
\left(
\frac{1/(2\pi)^2}{s^4+s^2 (2\sigma_{t,i}^2+ 2\alpha_{t,\ell}^2)}
\right)^{d/2}
\exp\left( \frac{1}{2} 
    \frac{(\sigma_{t,i}^2+ \alpha_{t,\ell}^2) s^2}{s^2+2\sigma_{t,i}^2+2\alpha_{t,\ell}^2}
    \Xi
\right).
\end{align*}
\end{lemma}
\begin{proof}
To ease our notation, we first let 
\begin{align*}
    \overline{b}_{t,i,\ell}^k(x) = -(2s^2/C_2) \cdot 
    \phi(x;\theta_\ell,(\sigma_{t,i}^2+\alpha_{t,\ell}^2)^{1/2}) (x -
    \theta_\ell)_k.
\end{align*}
It follows that
\begin{align}
    &\quad \langle \overline{b}_{t,i,\ell}^k\partial_k\phi_r,\phi_c\rangle \\
    &= 
    -(2s^2/C_2) \frac{1}{-2(1-\omega_{t,i})\sqrt{(2\pi\sigma_{t,i}^2)^d}
    \sigma_{t,i}^2/(\sigma_{t,i}^2+\alpha_{t,\ell}^2)}
    \langle b_{t,i,\ell}^k\partial_k\phi_r,\phi_c\rangle \nonumber\\
    &= 
    \frac{(s^2/C_2)}{(1-\omega_{t,i})\sqrt{(2\pi\sigma_{t,i}^2)^d}
    \sigma_{t,i}^2/(\sigma_{t,i}^2+\alpha_{t,\ell}^2)}
    \langle b_{t,i,\ell}^k\partial_k\phi_r,\phi_c\rangle.
    \label{eqn:lemmab.4.part0}
\end{align}
We compute $\langle \overline{b}_{t,i,\ell}^k\partial_k\phi_r,\phi_c\rangle$
instead of directly working with \eqref{eqn:btiellk2}.  
First, we observe that 
\begin{align*}
    {\langle x_k^2\rangle_{\ell,r,c} - \langle x_k\rangle_{\ell,r,c}^2} =
    \frac{1}{C_1} = 
    \frac{(\sigma_{t,i}^2+ \alpha_{t,\ell}^2)
    s^2}{s^2+2\sigma_{t,i}^2+2\alpha_{t,\ell}^2},
\end{align*}
and that
\begin{align*}
    &\quad\frac{1}{C_1^2 (\sigma_{t,i}^2+\alpha_{t,\ell}^2) s^4}\\
    &=\left(
    \frac{(\sigma_{t,i}^2+ \alpha_{t,\ell}^2) s^2}{s^2+2\sigma_{t,i}^2+2\alpha_{t,\ell}^2}
    \right)^2
    \frac{1}{(\sigma_{t,i}^2+\alpha_{t,\ell}^2) s^4}\\
    &=
    \frac{\sigma_{t,i}^2+\alpha_{t,\ell}^2}{(s^2+2\sigma_{t,i}^2+2\alpha_{t,\ell}^2)^2}.
\end{align*}
Using Lemma
\ref{lem:algebra-formula} on the third equality, we see that
\begin{align}
&\quad \int \overline{b}_{t,i,\ell}^k(x) \partial_k \phi_r(x) \phi_c(x) dx
\nonumber \\
&=\int -2s^2/C_2 \phi(x;\theta_\ell,(\sigma_{t,i}^2+\alpha_{t,\ell}^2)^{1/2}) (x - \theta_\ell)_k
\left(-\frac{1}{s^2}(x-\theta_r)_k \phi_r(x)\right) \phi_c(x) dx \nonumber \\
&=(2/C_2)\int \phi_r(x) \phi_c(x) \phi(x;\theta_\ell,(\sigma_{t,i}^2+\alpha_{t,\ell}^2)^{1/2}) 
(x - \theta_\ell)_k (x-\theta_r)_k dx \nonumber \\
&= 2\int h_{\ell,r,c}(x) (x_k^2 - x_k (\theta_\ell + \theta_r)_k +
\theta_{\ell,k} \theta_{r,k}) dx \nonumber \\
&= 2\left( 
 {\langle x_k^2\rangle_{\ell,r,c} - \langle x_k\rangle_{\ell,r,c}^2}
+ (\langle x_k \rangle_{\ell,r,c} - \theta_{\ell,k})
(\langle x_k\rangle_{\ell,r,c} - \theta_{r,k}) \right). \nonumber
\end{align}
Continuing with the calculation, 
\begin{align}
&\quad    \langle \overline{b}_{t,i,\ell}^k \partial_k \phi_r, \phi_c \rangle
\nonumber \\
&= 2\left(
\frac{1}{C_1} 
+ 
\frac{1}{C_1^2} 
\left(
\frac{ \theta_r - \theta_\ell }{s^2}
+
\frac{ \theta_c-\theta_\ell}{s^2}
\right)_k
\left(
\frac{\theta_\ell - \theta_r}{\sigma_{t,i}^2+\alpha_{t,\ell}^2}
+
\frac{\theta_c-\theta_r }{s^2}
\right)_k
\right) \nonumber \\
&= 
2
\left(
\frac{1}{C_1}
+ 
\frac{1}{C_1^2}
\frac{
\left(-2 \theta_\ell + \theta_r + \theta_c \right)_k
\left(
s^2 \theta_\ell
-\left(
\sigma_{t,i}^2+\alpha_{t,\ell}^2 + s^2
\right)
\theta_r
+
(\sigma_{t,i}^2+\alpha_{t,\ell}^2) {\theta_c}
\right)_k
}{(\sigma_{t,i}^2+\alpha_{t,\ell}^2) s^4 }\right) \nonumber \\
&=
\frac{(2\sigma_{t,i}^2+2\alpha_{t,\ell}^2)s^2}{s^2+2\sigma_{t,i}^2+2\alpha_{t,\ell}^2}
\label{eqn:lemmab.4.part1}\\
&\quad +
\frac{2\sigma_{t,i}^2+2\alpha_{t,\ell}^2}{(s^2+2\sigma_{t,i}^2+2\alpha_{t,\ell}^2)^2}
\label{eqn:lemmab.4.part2}\\
&\quad \cdot \bm 1^\top
\left(
\left(
\begin{bmatrix}
    s^2 \\
    -(\sigma_{t,i}^2+\alpha_{t,\ell}^2+s^2) \\
    \sigma_{t,i}^2+\alpha_{t,\ell}^2
\end{bmatrix}
\begin{bmatrix}
    -2 & 1 & 1
\end{bmatrix}
\right) * 
\left(
\begin{bmatrix}
    \theta_{\ell,k} \\
    \theta_{r,k} \\
    \theta_{c,k} 
\end{bmatrix}
\begin{bmatrix}
    \theta_{\ell,k}, \theta_{r,k}, \theta_{c,k} 
\end{bmatrix}
\right)
\right) \bm 1. \label{eqn:lemmab.4.part3}
\end{align}
Putting together 
\eqref{eqn:lemmab.4.part1},
\eqref{eqn:lemmab.4.part2},
\eqref{eqn:lemmab.4.part3} and 
\eqref{eqn:lemmab.4.part0}, 
and plugging in the full expression 
for $C_1$, we see that
\begin{align*}
&\quad \langle b_{t,i,\ell}^k\partial_k\phi_r,\phi_c\rangle\\
&= (2\pi\sigma_{t,i}^2)^{d/2}
    \frac{(1-\omega_{t,i}) \sigma_{t,i}^2 C_2}{s^2(\sigma_{t,i}^2+\alpha_{t,\ell}^2) }
    \left(
    \frac{s^2(2\sigma_{t,i}^2+2\alpha_{t,\ell}^2)}{s^2+2\sigma_{t,i}^2+2\alpha_{t,\ell}^2}
+
\frac{2\sigma_{t,i}^2+2\alpha_{t,\ell}^2}{(s^2+2\sigma_{t,i}^2+2\alpha_{t,\ell}^2)^2}
\xi \right).
\end{align*}
Our claim follows after summing over $k$ and replacing $C_2$ with its full
expression. 
\end{proof}

%%%%%%%% LEMMA B.5
\begin{lemma}\label{lem:a}
    For each $\ell,r,c$, $t$ and $i$, 
\begin{align*}
    &\sum_{k_1=1}^{d} \sum_{k_2=1}^d  
    \langle a_{t,i,\ell}^{k_1,k_2} \partial_{k_1} \partial_{k_2} \phi_r,
    \phi_c \rangle\\
    &= \frac{1}{s^4} (1-\omega_{t,i})^2 (2\pi\sigma_{t,i}^2)^{d/2} 
    \left(
    \left(\frac{\sigma_{t,i}^2\alpha_{t,\ell}^2}{\sigma_{t,i}^2 + \alpha_{t,\ell}^2}\right)
\sum_k
\begin{bmatrix}
    \langle x_k^2\rangle_{\ell,r,c} &
    \langle x_k\rangle_{\ell,r,c} &
    1
\end{bmatrix}
\begin{bmatrix}
1 \\ -2\theta_{r,k} \\ \theta_{r,k}^2 - s^2
\end{bmatrix}
\right. \\
&+ 
\left(\frac{\sigma_{t,i}^2}{\sigma_{t,i}^2 + s^2}\right)^2
\bm 1^\top 
\left(
\sum_{k}
\begin{bmatrix}
\langle x_{k}^4 \rangle_{\ell,r,c} &\langle x_{k}^3 \rangle_{\ell,r,c} &\langle x_{k}^2 \rangle_{\ell,r,c} \\
\langle x_{k}^3 \rangle_{\ell,r,c} &\langle x_{k}^2 \rangle_{\ell,r,c} &\langle x_{k} \rangle_{\ell,r,c}  \\
\langle x_k^2 \rangle_{\ell,r,c} &\langle x_{k} \rangle_{\ell,r,c} & 1
\end{bmatrix}
*
\begin{bmatrix}
    1 & -2 \theta_{r,k}  & \theta_{r,k}^2 -s^2 \\
    -2 \theta_{\ell,k} & 4 \theta_{\ell,k} \theta_{r,k} & -2\theta_{\ell,k} (\theta_{r,k}^2 -s^2) \\
    \theta_{\ell,k}^2 & -2 \theta_{r,k} \theta_{\ell,k}^2 &  (\theta_{r,k}^2 -s^2) \theta_{\ell,k}^2 
\end{bmatrix}\right.\\
&+\left.\left.
    \sum_{k_1}\sum_{k_2\neq k_1}
\begin{bmatrix}
    \langle x_{k_1}^2 \rangle_{\ell,r,c} \\
    \langle x_{k_1} \rangle_{\ell,r,c} \\
    1 
\end{bmatrix}
\begin{bmatrix}
    \langle x_{k_2}^2 \rangle_{\ell,r,c}&\langle  x_{k_2} \rangle_{\ell,r,c}& 1
\end{bmatrix}
*
\begin{bmatrix}
    1 \\
    -(\theta_\ell + \theta_r)_{k_1} \\
    \theta_{\ell,k_1} \theta_{r,k_1}
\end{bmatrix}
\begin{bmatrix}
    1 & -(\theta_\ell +\theta_r)_{k_2} & \theta_{\ell,k_2} \theta_{r,k_2}
\end{bmatrix}
\right) \bm 1 \right)\\
&\qquad \cdot
\left(
\frac{1/(2\pi)^2}{s^4+s^2 (2\sigma_{t,i}^2+ 2\alpha_{t,\ell}^2)}
\right)^{d/2}
\exp\left( \frac{1}{2} 
    \frac{(\sigma_{t,i}^2+ \alpha_{t,\ell}^2) s^2}{s^2+2\sigma_{t,i}^2+2\alpha_{t,\ell}^2}
    \Xi
\right).
\end{align*}
\end{lemma}
\begin{proof}
    Note that 
\begin{align*}
    a_{t,\ell}^{k_1,k_2}(x) 
        =
        (1-\omega_{t,i})^2 (2\pi\sigma_{t,i}^2)^{d/2}
        \left(
    \left(\frac{\sigma_{t,i}^2}{\sigma_{t,i}^2+\alpha_{t,\ell}^2}\right)^2 
        A_{t,\ell}^{k_1,k_2}(x) + 
        \left(\frac{\sigma_{t,i}^2\alpha_{t,\ell}^2}{\sigma_{t,i}^2+\alpha_{t,\ell}^2}\right)
        B_{t,\ell}^{k_1,k_2}(x)\right),
    \end{align*}
where
\begin{align*}
    &A_{t,\ell}^{k_1,k_2}(x) 
    =
    \phi(x;\theta_\ell,(\sigma_{t,i}^2+\alpha_{t,\ell}^2)^{1/2})
    (x-\theta_\ell)_{k_1} (x-\theta_\ell)_{k_2},\\
    &B_{t,\ell}^{k_1,k_2}(x) 
    =
    \phi(x;\theta_\ell,(\sigma_{t,i}^2+\alpha_{t,\ell}^2)^{1/2})
    \bm 1\{k_1=k_2\}.
\end{align*}
We first compute the diagonal terms, i.e., the $k_1=k_2$ cases.  Note that
\begin{align*}
&\qquad
\sum_{k_1} \sum_{k_2} \int B_{t,\ell}^{k_1,k_2}(x) \partial_{k_1,k_2}^2 \phi_r(x) \phi_c(x) dx \\
&=
\sum_{k_1} \sum_{k_2} \int 
    \phi(x;\theta_\ell,(\sigma_{t,i}^2+\alpha_{t,\ell}^2)^{1/2})
    \bm 1\{k_1=k_2\}
\partial_{k_1,k_2}^2 \phi_r(x) \phi_c(x) dx \\
&=
\sum_{k} \int 
    \phi(x;\theta_\ell,(\sigma_{t,i}^2+\alpha_{t,\ell}^2)^{1/2})
    \partial_{k,k}^2 \phi_r(x) \phi_c(x) dx \\
&=
\frac{1}{s^4}
\sum_{k} \int 
    \phi(x;\theta_\ell,(\sigma_{t,i}^2+\alpha_{t,\ell}^2)^{1/2})
    (x_k^2 - 2\theta_{r,k}x_k + \theta_{r,k}^2-s^2)
    \phi_r(x) \phi_c(x) dx \\
&=
\frac{C_2}{s^4} 
\sum_{k} \int 
    h_{\ell,r,c}(x)
    (x_k^2 - 2\theta_{r,k}x_k + \theta_{r,k}^2-s^2)dx.
\end{align*}
and also that
\begin{align*}
&\qquad
\sum_{k_1} \sum_{k_2} \int \bm 1\{k_1=k_2\} A_{t,\ell}^{k_1,k_2}(x) \partial_{k_1,k_2}^2 \phi_r(x) \phi_c(x) dx \\
&=
\sum_{k_1} \sum_{k_2} \int 
    \phi(x;\theta_\ell,(\sigma_{t,i}^2+\alpha_{t,\ell}^2)^{1/2})
    \bm 1\{k_1=k_2\}
    (x-\theta_\ell)_{k_1}(x-\theta_\ell)_{k_2}
\partial_{k_1,k_2}^2 \phi_r(x) \phi_c(x) dx \\
&=
\sum_{k} \int 
    \phi(x;\theta_\ell,(\sigma_{t,i}^2+\alpha_{t,\ell}^2)^{1/2})
    (x-\theta_\ell)_{k}^2
    \partial_{k,k}^2 \phi_r(x) \phi_c(x) dx \\
&=
\frac{1}{s^4}
\sum_{k} \int 
    \phi(x;\theta_\ell,(\sigma_{t,i}^2+\alpha_{t,\ell}^2)^{1/2})
    (x-\theta_\ell)_{k}^2
    (x_k^2 - 2\theta_{r,k}x_k + \theta_{r,k}^2-s^2)
    \phi_r(x) \phi_c(x) dx \\
&=
\frac{C_2}{s^4} 
\sum_{k} \int 
    h_{\ell,r,c}(x)
    (x_k^2-2\theta_{\ell,k} x_k + \theta_{\ell,k}^2)
    (x_k^2 - 2\theta_{r,k}x_k + \theta_{r,k}^2-s^2)dx.
\end{align*}

Next, we compute the off-diagonal terms, i.e., the $k_1\neq k_2$ cases. First,
using our calculation just above, we see that
we note that
\begin{align*}
&\quad \sum_{k_1} \sum_{k_2\neq k_1} \int A_{t,\ell}^{k_1,k_2}(x)
    \partial_{k_1,k_2}^2 \phi_r(x) \phi_c(x) dx \\
    &= \frac{C_2}{s^4}
    \sum_{k_1} \sum_{k_2\neq k_1} \int h_{\ell,r,c}(x)
    (x-\theta_\ell)_{k_1}
    (x-\theta_\ell)_{k_2}
    (x-\theta_r)_{k_1}
    (x-\theta_r)_{k_2}
    dx.
\end{align*}
Our claim follows from this after combining them together, and simplifying
the combined term into a matrix notation.
\end{proof}

\section{Proof for Theorem \ref{thm:exactposterior}}
Here, we will take the convention that $\bm X_t = (X_{t,ik})_{i=1,k=1}^{n,d}$ 
is organized as a matrix.  By the $i$-th row of $\bm X_t$, we mean $X_{t,i}= 
(X_{t,1},\ldots, X_{t,d})$.
Let 
\begin{align*}
\varphi_t(\bm v)  = \mathbb P\left[e^{\imath \langle \bm v, \bm X_t \rangle}
\left| \mathcal F_t\right.\right]
= \int {\rho}_t(d\bm x) e^{\imath \langle \bm v, \bm x \rangle},
\end{align*}
where for each $\bm v$ and $\bm x$, 
\begin{align*}
    \langle \bm v, \bm x \rangle \equiv \sum_{i=1}^n \sum_{k=1}^d v_{ik} x_{ik}. 
\end{align*}
In other words, $\varphi_t$ is the (random) conditional characteristic function of $\bm X_t$.
Note
\begin{align*}
    p_t(\bm y) = \frac{1}{2\pi} \int e^{-\imath \langle \bm v, \bm y\rangle}
    \varphi_t(\bm v) d\bm v.
\end{align*}
Also, let, for each $\bm v$ and $\bm x \in \mathbb R^{n\times d}$, 
\begin{align}
    a_t(\bm v|\bm x) \equiv \lim_{\varepsilon\rightarrow0} \frac{1}{\varepsilon}
    \mathbb E[e^{\imath \langle\bm v, \bm X_{t+\varepsilon}- \bm X_t \rangle}
    - 1 \left| \bm X_t =\bm x \right.].
\end{align}
For each $f \in \mathcal{B}(\mathbb X)$,
$f_{-i}$ denotes the function obtained by 
fixing all other indices different from the $i$-th actor indices
but letting the $i$-th actor indices to be free, and 
if $f_{-i}$ is in the domain of the operator $\mathcal{A}(\mu_t)$, 
with some abuse of notation, we write:
\begin{align*}
    \mathcal{A}(\mu_t)f(\bm z) 
    =
    \sum_{i=1}^n (\mathcal{A}(\mu_t)f_{-i}(\bm z))(z_i).
\end{align*}
Similarly, for each $v \in\mathbb R^d$, let
\begin{align*}
    \varphi_{t,i}(v) 
    = \mathbb E\left[e^{\imath \langle v, X_i(t)\rangle}\left| \mathcal F_t\right.\right].
\end{align*}
In other words, $\varphi_{t,i}$ denotes the conditional characteristic
function of the $i$-th row $X_k(t)$ of $\bm X(t)$, and also, 
let, for $v\in \mathbb R^d$, and $x\in \mathbb X$, 
\begin{align}
    a_{t}(v|x) \equiv \lim_{\varepsilon\rightarrow0} \frac{1}{\varepsilon}
    \mathbb E[e^{\imath \langle v, X_{i, t+\varepsilon}- X_{t,i}\rangle } - 1 \left|
    X_{t,i} = x \right.].
\end{align}
Note that the definition of $a_t(v|x)$ is actually independent of a particular
choice of vertex $i$ as they are all identically distributed. 

One can prove the next result by directly following \citet{DonaldSnyder75},
but one needs to adapt to the fact that the underlying process can now be a
time-inhomogeneous non-linear Markov process.  The proof details are left to
the reader.    For a survey of similar techniques, see also
\cite{KunitaSFandSDE1997} and \cite{BainCrisan}. 
\begin{proposition}\label{prop:DonaldSnyder}
    For each ${\bm v} \in \mathbb R^{n\times d}$ and $t\in (0,\infty)$, 
\begin{align*}
    d\varphi_t(\bm v) 
    =
    \langle \rho_t, e^{\imath \langle\bm v, \cdot \rangle} a_t(\bm v|\cdot)\rangle dt
    + 
    \bm 1^\top 
    \langle \rho_t,  e^{\imath \langle \bm v, \cdot  \rangle} (\widetilde{\bm
    \lambda} -\bm1\bm1^\top )^\prime\rangle
    d\bm{M}_t
    \bm 1.
\end{align*}
\end{proposition}

Our proof of Theorem \ref{thm:exactposterior} is by brute force calculation,
starting from Proposition \ref{prop:DonaldSnyder}. 
In particular, our claim in Theorem \ref{thm:exactposterior} follows from
Proposition \ref{prop:DonaldSnyder} by directly applying
Lemma \ref{lemmaAA}, Lemma \ref{lemmaAB} and Lemma \ref{lemmaAC} which we
list and prove now. 
\begin{lemma}\label{lemmaAA} 
For each $t \in [0,\infty)$, 
\begin{align*}
    a_{t}(v|x) =
    \mathcal A(\mu_t)e^{\imath \langle v, \cdot-x\rangle}\left(x\right).
\end{align*}
\end{lemma}
\begin{proof}
Fix $t$, $i$, $v$ and $x$. Then, for each $\varepsilon > 0$, we have:
\begin{align*}
    \mathbb{E}\left[e^{\imath \langle v, X_{i, t+\varepsilon}-x\rangle }\left|X_{t,i} = x\right.\right]
    =
    1 + 
    \int_{0}^{\varepsilon}\mathbb{E}\left[\mathcal{A}\left(\mu_{t+s}\right)e^{\imath
    \langle v, \cdot-x\rangle}\left(X_{t+s,i}\right)\left|X_{t,i}=x\right.\right]ds.
\end{align*}
We have
\begin{align*}
    \sup_{y\in\mathbb X}
    \left|
    \mathcal{A}\left(\mu_{t+s}\right)e^{\imath \langle v,  \cdot-x\rangle}(y) 
    \right|
    \le 
    \left|e^{\imath \langle v, \cdot-x\rangle}\right| = 1,
\end{align*}
and hence, 
\begin{align*}
    \left|
    \mathbb{E}\left[\mathcal{A}\left(\mu_{t+s}\right)
    e^{\imath \langle v, \cdot-x\rangle}(X_{t+s})
    \left|X_{t}=x\right.
    \right]
    \right| 
    \le 1.
\end{align*}
It follows that 
\begin{align*}
    &\ 
    \lim_{\varepsilon\downarrow 0}
    \frac{1}{\varepsilon} 
    \int_{0}^{\varepsilon} 
    \mathbb{E}\left[\mathcal{A}\left(\mu_{t+s}\right)
    e^{\imath \langle v, \cdot-x\rangle}(X_{t+s})
    \left|X_{t}=x\right.
    \right]
    ds \\
    = &\
    \mathbb{E}\left[\mathcal{A}\left(\mu_{t}\right)
    e^{\imath \langle v, \cdot-x\rangle}(X_{t})
    \left|X_{t}=x\right.
    \right]\\
    = &\ 
    \mathcal{A}\left(\mu_{t}\right)
    e^{\imath \langle v, \cdot-x\rangle}(x).
\end{align*}
\hfill\end{proof}
%\begin{proof}{Proof}
%\hfill$\Box$\qquad\end{proof}
\begin{lemma}\label{lemmaAB}
For each $f \in C_b(\mathbb X)$, we have:
    \begin{align*}
        \int_{\mathbb X} 
        f(\bm y) 
        \left(
        \frac{1}{2\pi} 
        \int
        e^{ -\imath \langle \bm v, \bm y \rangle }
        \langle \rho_t, e^{\imath \langle \bm v, \cdot
        \rangle}a_t\left(\bm v|\cdot\right)\rangle 
        d\bm v
        \right)
        d\bm y
        = 
        \int_{\mathbb X } 
        \left(\mathcal{A}\left(\mu_t\right)f\right)(\bm z)\rho_t(d\bm z).
\end{align*}
\end{lemma}
\begin{proof}
Fix $\bm y \in \mathbb X$ and note:
\begin{align*}
    h(\bm y) 
    &\equiv
    \frac{1}{2\pi} 
        \int
        e^{ -\imath \langle \bm v, \bm y \rangle }
        \langle \rho_t, e^{\imath \langle \bm v, \cdot
        \rangle}a_t\left(\bm v|\cdot\right)\rangle 
    d\bm v\\
    &= 
    \frac{1}{2\pi}
        \int
        e^{ -\imath \langle \bm v, \bm y \rangle }
        \int 
        \rho_t(d\bm z) 
        e^{\imath \langle \bm v, \bm z \rangle}a_t\left(\bm v|\bm z\right)
        d\bm v\\
    &= 
        \int \rho_t(d\bm z) 
        \left(
        \frac{1}{2\pi}
        \int
        e^{ -\imath \langle \bm v, \bm y \rangle }
        e^{\imath \langle \bm v, \bm z \rangle}a_t\left(\bm v|\bm z\right)
        d\bm v
        \right),
\end{align*}
and that
\begin{align*}
    \frac{1}{2\pi} \int e^{\imath \langle \bm v, \bm z -\bm y
    \rangle}a_t\left(\bm v|\bm z\right) d\bm v
    &= 
    \frac{1}{2\pi}
    \int e^{\imath \langle \bm v, \bm z -\bm y \rangle}
    \lim_{\varepsilon\rightarrow0} 
    \frac{1}{\varepsilon}
    \mathbb{E}\left[e^{\imath \langle \bm v, \bm X_{t+\varepsilon} - \bm
    z\rangle}-1\left|\bm X_t = \bm z\right. \right]
    d\bm v\\
    &= 
    \frac{1}{2\pi}
    \int
    \lim_{\varepsilon\rightarrow0} 
    \frac{1}{\varepsilon}
    \mathbb{E}\left[e^{\imath \langle \bm v, \bm X_{t+\varepsilon} - \bm y\rangle}-
    e^{\imath \langle \bm v, \bm z -\bm y \rangle}
    \left|\bm X_t = \bm z\right. \right]
    d\bm v.
\end{align*}
Treating $h(\bm y)$ as a generalized function (i.e.\ a tempered
distribution), we have:
\begin{align*}
    &\ \int h(\bm y) f(\bm y) d\bm y \\ 
    &=
    \int    
    \rho_t(d\bm z) 
    \left(
    \int
    f(\bm y)
    \left(
    \frac{1}{2\pi}
    \int
    \lim_{\varepsilon\rightarrow0} 
    \frac{1}{\varepsilon}
    \mathbb{E}\left[e^{\imath \langle \bm v, \bm X_{t+\varepsilon} - \bm y\rangle}-
    e^{\imath \langle \bm v, \bm z -\bm y \rangle}
    \left|\bm X_t = \bm z\right. \right]
    d\bm v
    \right)
    d\bm y
    \right)\\
    &=
    \int    
    \rho_t(d\bm z) 
    \lim_{\varepsilon\rightarrow0} 
    \frac{1}{\varepsilon}
    \left(
    \mathbb{E}\left[
    \int f(\bm y)
    \left(
    \frac{1}{2\pi}
    \int
    e^{\imath \langle \bm v, \bm X_{t+\varepsilon} - \bm y\rangle}
    d\bm v
    \right)
    d\bm y
    \left|\bm X_t = \bm z\right. \right]
    -
    \int f(\bm y)
    \left(
    \frac{1}{2\pi}
    \int
    e^{\imath \langle \bm v, \bm z -\bm y \rangle}
    d\bm v
    \right)
    d\bm y
    \right)\\
    &=
    \int    
    \rho_t(d\bm z) 
    \lim_{\varepsilon\rightarrow0} 
    \frac{1}{\varepsilon}
    \left(
    \mathbb{E}\left[
    \int f(\bm y)
    \delta_0(\bm X_{t+\varepsilon} - \bm y)
    d\bm y
    \left|\bm X_t = \bm z\right. \right]
    -
    \left(
    \int f(\bm y)
    \delta_0(\bm z -\bm y )
    d\bm y
    \right)
    \right)\\
    &=
    \int    
    \rho_t(d\bm z) 
    \lim_{\varepsilon\rightarrow0} 
    \frac{1}{\varepsilon}
    \left(
    \mathbb{E}\left[
    f(\bm X_{t+\varepsilon})
    \left|\bm X_t = \bm z\right. \right]
    -
    f(\bm z)
    \right)\\
    &=
    \int    
    \rho_t(d\bm z) 
    (\mathcal{A}\left(\mu_t\right)f)(\bm z).
\end{align*}
\hfill\end{proof}
\begin{lemma}\label{lemmaAC}
    For each $f \in C_b(\mathbb X)$, we have:
    \begin{align*}
        \int_{\mathbb X} 
        f(\bm y) 
        \left(
        \frac{1}{2\pi} 
        \int
        e^{ -\imath \langle \bm v, \bm y \rangle }
        \langle \rho_t, e^{\imath \langle \bm v, \cdot
        \rangle} (\bm \lambda(\cdot) - \bm 1 \bm1^\top) \rangle 
        d\bm v
        \right)
        d\bm y
        = 
        \int_{\mathbb X } 
        \rho_t(d\bm z)
        f(\bm z) \left(\bm \lambda(\bm z) - \bm 1 \bm1^\top \right).
    \end{align*}
\end{lemma}
\begin{proof}
    Note
\begin{align*}
    &\    \int_{\mathbb X} 
        f(\bm y) 
        \left(
        \frac{1}{2\pi} 
        \int
        e^{ -\imath \langle \bm v, \bm y \rangle }
        \langle \rho_t, e^{\imath \langle \bm v, \cdot
        \rangle} (\bm \lambda(\cdot) - \bm 1 \bm1^\top) \rangle 
        d\bm v
        \right)
        d\bm y\\
    =&
    \int    
    \rho_t(d\bm z) 
    \left(\bm \lambda(\bm z) - \bm 1 \bm1^\top\right)
    \left(
    \int
    f(\bm y)
    \left(
    \frac{1}{2\pi}
    \int
    e^{\imath \langle \bm v, \bm z -\bm y \rangle}
    d\bm v
    \right)
    d\bm y
    \right)\\
    =&
    \int    
    \rho_t(d\bm z) 
    \left(\bm \lambda(\bm z) - \bm 1 \bm1^\top\right)
    \left(
    \int
    f(\bm y)
    \delta_0(\bm z - \bm y)
    d\bm y
    \right)\\
    =&
    \int    
    \rho_t(d\bm z) 
    \left(\bm \lambda(\bm z) - \bm 1 \bm1^\top\right)
    f(\bm z).
\end{align*}
\hfill\end{proof}\\

\section{Proof of Theorem \ref{thm:simplifyingposteriorupdaterule}}
Recall that for each $f \in \mathcal{B}(\mathbb X)$, $f_{-i}$ denotes the
function obtained by fixing all other indices different from the $i$-th actor
indices but letting the $i$-th actor indices to be free.  Fix $u \in
\{1,\ldots,n\}$.  Let $f \in \mathcal{B}(\mathbb X)$ be such that $f(\bm z) =
f_{-u}(z_u)$ for all $\bm z \in \mathbb X$.  For each $t\in (0,\infty)$, 
\begin{align*} 
    d\langle\rho_t, f \rangle
    &=
    \langle \rho_{t,u}, \mathcal{A}(\mu_t)f_{-u}\rangle
    dt \\
    &+
    \sum_{i<j,i\neq u, j\neq u } \int_{\mathbb X } 
    \rho_{t,u,i,j}(dz_u,dz_i,dz_j)
    f_{-u}(z_u) 
    \left(
    \frac{p_{t,j}(z_i)}{\langle p_{t,i},p_{t,j}\rangle} - 1
    \right) dM_{t,ij}\\
    &\ \ + 
    \sum_{i\neq u} \int_{\mathbb X } 
    \rho_{t,i,u}(dz_i,dz_u)
    f_{-u}(z_u)\left(\frac{p_{t,j}(z_i)}{\langle
    p_{t,i},p_{t,j}\rangle}-1\right)dM_{t,iu}.
\end{align*}
Then, the claimed formula follows from our assumption in \eqref{productjointassumption}.

\section{Proof of Theorem
\ref{thmstat:continuousembedding}}\label{thmstat:contembedding:proof}
Suppose that $M_\varepsilon \rightarrow M_0$ as $\varepsilon \rightarrow 0$
and that for each $\varepsilon \ge 0$,  $M_\varepsilon$ satisfies the rank
condition, i.e., $\varrho(M_\varepsilon)$ is of rank at least $d$. 
Note that each $\xi_d(M_\varepsilon)$ is a non-empty compact subset of
$\mathbb R^{n\times d}$ since $ \|X_{\varepsilon,+}Q\|_F^2 =
\|X_{\varepsilon,+}\|_F^2$ for any
real orthogonal matrix $Q$. 
In particular, for sufficiently small $\varepsilon_0$, we may assume that 
$\sup_{\varepsilon \in [0,\varepsilon_0]} \| \xi_d^*(M_\varepsilon) \|_F^2 < \infty$. 
It is enough to show that 
for each arbitrary convergent subsequence of $\{\xi_d^*(M_\varepsilon)\}$,
\begin{align}
    \label{eqn:convoptsoln}
\lim_{\varepsilon\rightarrow0} \xi_d^*(M_\varepsilon) = \xi_d^*(M_0).
\end{align}
Consider an arbitrary convergent subsequence of $\{\xi_d^*(M_\varepsilon)\}$.
We begin by observing some linear algebraic facts.
First, any sequence of real orthogonal matrices has a convergent subsequence whose 
limit is also real orthogonal. 
Next, since both $X_{\varepsilon,+}$ and $Z$ are of rank $d$, 
there exists a unique real orthogonal $d\times d$ matrix $Q_{\varepsilon,+}$
such that 
\begin{align*}
    \xi_d^*(M_\varepsilon) = X_{\varepsilon,+} Q_{\varepsilon,+},
\end{align*}
and in fact, $Q_{\varepsilon,+} = U_{\varepsilon,+}V_{\varepsilon,+}^\top$ where
$X_{\varepsilon,+}^\top Z = U_{\varepsilon,+}S_{\varepsilon,+}V_{\varepsilon,+}^\top $
is a singular value decomposition of $X_{\varepsilon,+}^\top Z$, and 
$U_{\varepsilon,+}V_{\varepsilon,+}^\top $ is the corresponding \emph{unique right
factor} in the polar decomposition of $X_{\varepsilon,+}^\top Z$. 
Note that this implies the well-definition part of our
claim on $\xi_d^*$.
Also, since $M_\varepsilon \rightarrow M_0$, we have that 
\begin{align*}
    \lim_{\varepsilon\rightarrow 0} 
    \sum_{i=1}^{d} |\Sigma_{\varepsilon,ii} - \Sigma_{0,ii}|^2 \le
    \lim_{\varepsilon\rightarrow 0} \|M_\varepsilon - M_0\|_F^2 = 0. 
\end{align*}
For relevant linear algebra computation details for these facts,
see \citet[pg.\ 69, pg.\ 370, pg.\ 412, and pg.\ 431]{Horn_and_Johnson}.

Now, by taking a subsequence if necessary, we also have that 
for some $n\times d$ matrix $U_{*}$ such that $U_*^\top U_{*} = I$,
$\lim_{n\rightarrow \infty} U_{\varepsilon,+} = U_*$.
Then, 
\begin{align*}
    \lim_{\varepsilon \rightarrow 0 } 
    X_{\varepsilon,+} 
    = 
    \lim_{\varepsilon \rightarrow 0 } 
    U_{\varepsilon,+}{\Sigma_{\varepsilon,+}}^{1/2}
    =
    U_{*}{\Sigma_{0,+}}^{1/2} \equiv X_{*}.
\end{align*}
Next, note that  
if $\Sigma_{0,+}$ has distinct diagonal elements, 
then we also have $U_{*}=U_{0,+}$ so that $X_* = X_{0,+}$. 
On the other hand, more generally, i.e., even when there are some repeated diagonal elements, 
we can find a $d\times d$ matrix $Q_*$ such that $X_* = X_{0,+}Q_*$. 
To see this, note that 
the $i$-th column of $U_*$ is also an eigenvector of $\varrho(M_0)$ for the eigenvalue 
$\Sigma_{0,+,ii}$, and $U_*^\top U_*= I$, and hence it follows that for some 
$d\times d$ real orthogonal matrix $Q_*^\top $, we have $U_*Q_*^\top  = U_{0,+}$.
Moreover, exploiting the block structure of $\Sigma_{0,+}$ owing to 
algebraic multiplicity of eigenvalues, 
we can in fact choose $Q_*^\top $ so that $Q_* {\Sigma_{0,+}}^{1/2} = {\Sigma_{0,+}}^{1/2} Q_*$. 
Then, 
\begin{align*}
    X_* = U_*{\Sigma_{0,+}}^{1/2} = U_* Q_*^\top  Q_* {\Sigma_{0,+}}^{1/2} = U_{0,+}
    {\Sigma_{0,+}}^{1/2} Q_* = X_{0,+} Q_*.
\end{align*}

Now, we have
\begin{align*}
\xi_d^*(M_0)
&= X_{0,+} Q_{0,+} \\
&= X_* Q_*^\top Q_{0,+}\\
&= \lim_{\varepsilon\rightarrow 0} (X_{\varepsilon,+}Q_{\varepsilon,+} Q_{\varepsilon,+}^\top ) Q_*^\top  Q_{0,+} \\
&= \lim_{\varepsilon\rightarrow 0} \xi_d^*(M_\varepsilon)
(\lim_{\varepsilon\rightarrow 0}  Q_{\varepsilon,+}^\top ) Q_*^\top  Q_{0,+} \\
&= \lim_{\varepsilon\rightarrow 0} \xi_d^*(M_\varepsilon) \widetilde Q,
\end{align*}
where $\widetilde Q \equiv (\lim_{\varepsilon\rightarrow 0}
Q_{\varepsilon,+}^\top ) Q_*^\top  Q_{0,+}$ is a $d\times d$ real orthogonal matrix and 
and implicitly the limit was taken along a further subsequence when necessary.
Moreover, 
\begin{align*}
\| \lim_{\varepsilon \rightarrow 0}\xi_d^*(M_\varepsilon) - Z\|^2_F 
&\ge \| \xi_d^*(M_0) - Z\|_F^2\\
&= \| \lim_{\varepsilon\rightarrow0} \xi_d^*(M_\varepsilon) \widetilde Q- Z\|^2_F\\
&= \lim_{\varepsilon\rightarrow0}  \| \xi_d ^*(M_\varepsilon) \widetilde Q - Z\|^2_F\\
&\ge \lim_{\varepsilon\rightarrow0}  \| \xi_d^*(M_\varepsilon) - Z\|^2_F\\
&= \|  \lim_{\varepsilon\rightarrow0}  \xi_d^*(M_\varepsilon) - Z\|^2_F.
\end{align*}
In summary,  we have:
\begin{align*}
\|\lim_{\varepsilon \rightarrow 0} \xi_d(M_\varepsilon) - Z\|_F^2
=
\| \xi_d^*(M) - Z\|_F^2.
\end{align*}
By definition of $\xi_d^*(M_0)$, along with the facts that (i) all of the
convergent subsequences share the common limit, (ii) each subsequence has a
convergent subsequence, and (iii) $X_{0,+}$ and $Z$ have of full column rank
$d$, we have \eqref{eqn:convoptsoln}.
\end{appendix}

\bibliographystyle{plainnat}

%\bibliography{biblio}

\end{document}